\documentclass{article}

\usepackage{microtype}
\usepackage{graphicx}
\usepackage{subcaption}
\usepackage{booktabs} 
\usepackage{longtable}
\usepackage{multirow}
\usepackage{array}
\usepackage{caption}
\usepackage{tabularx}
\usepackage{makecell}

\usepackage{hyperref}

\usepackage[preprint]{icml2026}

\usepackage{amsmath}
\usepackage{amssymb}
\usepackage{mathtools}
\usepackage{amsthm}

\usepackage[capitalize,noabbrev]{cleveref}

\theoremstyle{plain}
\newtheorem{theorem}{Theorem}[section]
\newtheorem{proposition}[theorem]{Proposition}

\newtheorem{corollary}[theorem]{Corollary}
\theoremstyle{definition}

\theoremstyle{remark}

\usepackage[textsize=tiny]{todonotes}

\begin{document}

\twocolumn[
  \icmltitle{Beyond Parameter Arithmetic: Sparse Complementary Fusion for Distribution-Aware Model Merging}



  \icmlsetsymbol{equal}{*}

  \begin{icmlauthorlist}
    \icmlauthor{Weihong Lin}{comp}
    \icmlauthor{Lin Sun}{comp}
    \icmlauthor{Qilong Shi}{sch}
    \icmlauthor{Aomufei Yuan}{sch}
    \icmlauthor{Yuxuan Tian}{sch}
    \icmlauthor{Zhengyang Wang}{sch}
    \icmlauthor{Guangxiang Zhao}{comp}
    \icmlauthor{Xiangzheng Zhang}{comp}
    \icmlauthor{Tong Yang}{sch}
  \end{icmlauthorlist}

  \icmlaffiliation{comp}{Beijing Qiyuan Technology Co., Ltd., Beijing, China}
\icmlaffiliation{sch}{Peking University, Beijing, China}

  \icmlcorrespondingauthor{Lin Sun}{sunlin1@360.cn}

  \icmlkeywords{Machine Learning, ICML}

  \vskip 0.3in
]



\printAffiliationsAndNotice{}  

\begin{abstract}
  Model merging has emerged as a promising paradigm for composing the capabilities of large language models by directly operating in weight space, enabling the integration of specialized models without costly retraining. However, existing merging methods largely rely on parameter-space heuristics, which often introduce severe interference, leading to degraded generalization and unstable generation behaviors such as repetition and incoherent outputs.
In this work, we propose Sparse Complementary Fusion with reverse KL (SCF-RKL), a novel model merging framework that explicitly controls functional interference through sparse, distribution-aware updates. Instead of assuming linear additivity in parameter space, SCF-RKL measures the functional divergence between models using reverse Kullback-Leibler divergence and selectively incorporates complementary parameters. This mode-seeking, sparsity-inducing design effectively preserves stable representations while integrating new capabilities.
We evaluate SCF-RKL across a wide range of model scales and architectures, covering both reasoning-focused and instruction-tuned models. Extensive experiments on 24 benchmarks spanning advanced reasoning, general reasoning and knowledge, instruction following, and safety demonstrate, vision classification  that SCF-RKL consistently outperforms existing model merging methods while maintaining strong generalization and generation stability.
\end{abstract}

\section{Introduction}

\setlength{\intextsep}{0pt}
\begin{figure}[h]
\centering
\includegraphics[width=\columnwidth]{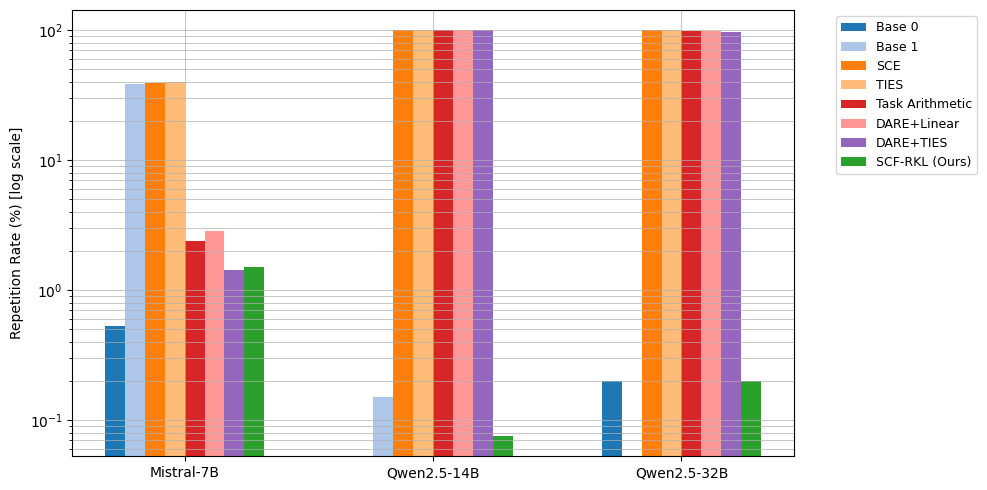}
\caption{Repetition rates across fusion methods on Mistral-7B, Qwen2.5-14B, and Qwen2.5-32B. Baseline merging methods severely amplify repetition even when base models exhibit near-zero rates, whereas SCF-RKL consistently maintains low repetition ($<$1.6\%) across all scales.}
\label{fig:all_repetitions}
\vspace{-0.5cm}
\end{figure}
\setlength{\intextsep}{0pt}
\begin{figure}[h]
\centering
\begin{subfigure}[b]{0.48\columnwidth}
    \centering
    \includegraphics[width=\linewidth]{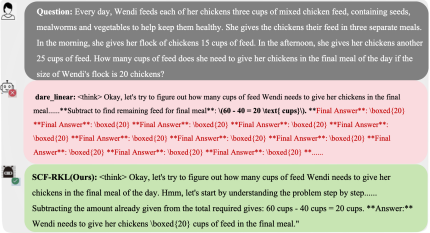}
    \label{fig:case_repetition}
    \vspace{-0.5cm}
\end{subfigure}
\hfill
\begin{subfigure}[b]{0.48\columnwidth}
    \centering
    \includegraphics[width=\linewidth]{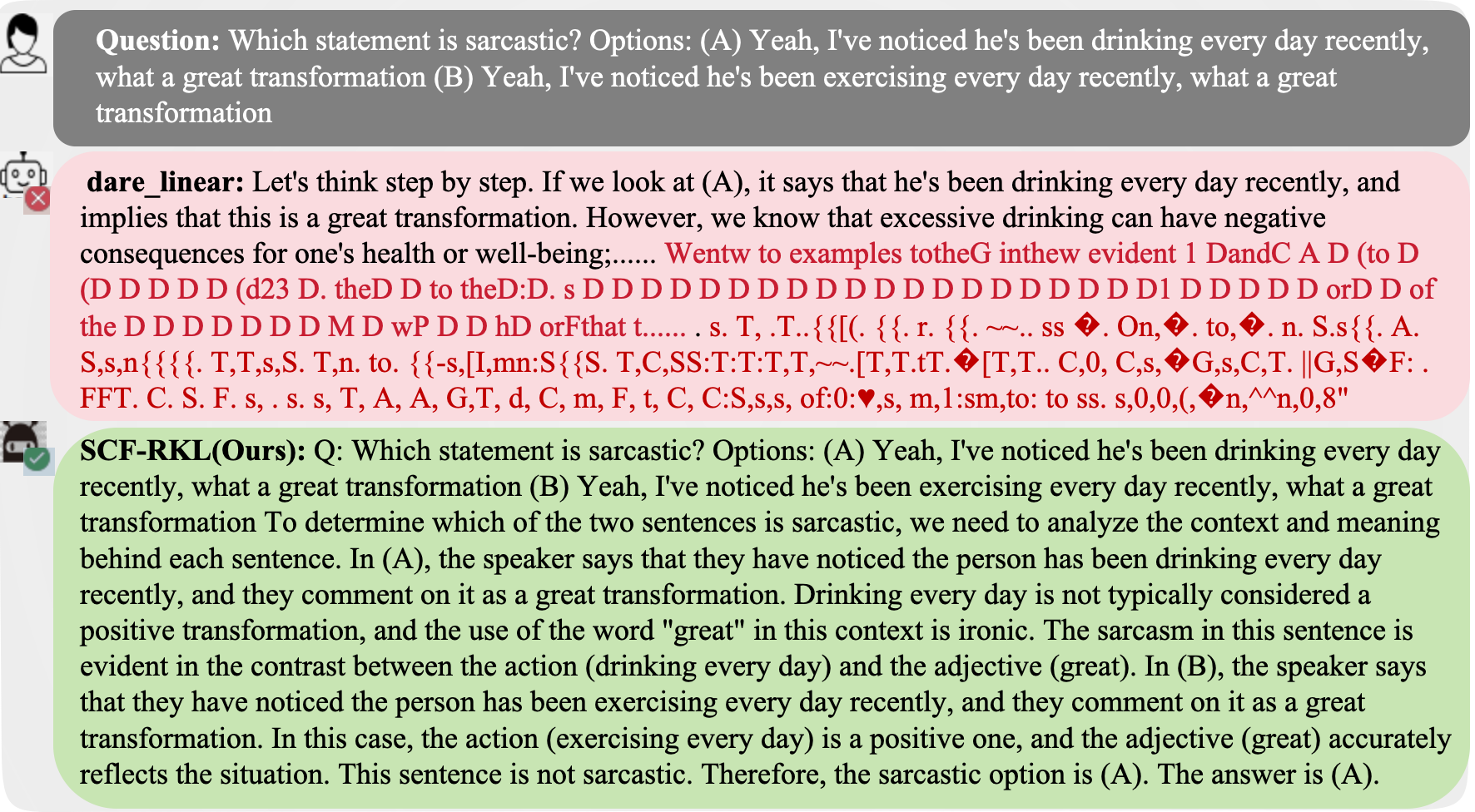}
    \label{fig:case_incoherent}
    \vspace{-0.5cm}
\end{subfigure}
\caption{Unstable generation in existing fusion: (Left) catastrophic repetition and (Right) incoherent outputs—whereas SCF-RKL preserves both correctness and coherence.}
\label{fig:motivating_cases}
\vspace{-0.5cm}
\end{figure}
\setlength{\intextsep}{0pt}
Large language models (LLMs) have been demonstrating unprecedented capabilities across a wide spectrum of tasks, fostering a rapidly growing ecosystem of specialized models  (e.g., Deepseek-R1\cite{guo2025deepseek}, Qwen2.5\cite{qwen2.5-tech-report}, Qwen3\cite{qwen3-tech-report}, LLaMA3\cite{llama3-herd} , Mistral\cite{mistaal7b} ) fine-tuned for reasoning\cite{wei2022chain}, instruction following\cite{wang2025light}, safety alignment\cite{si2026efficientswitchablesafetycontrol,tan2026triplayrltriroleselfplayreinforcement}, and domain-specific expertise. Model merging\cite{ilharco2022editing,yadav2023ties,du2024parameterPCB-merging,yang2023adamerging,wan2025fusechat,zeng2025robustmerge,chen2025bring,ma2025led} has recently emerged as a promising paradigm to integrate such complementary capabilities directly in weight space, offering a training-free alternative to costly retraining or continual finetuning. By composing pretrained or finetuned models, model merging enables practitioners to flexibly construct multi-capability systems, not only from publicly available open-source models, but also from custom-trained domain specialists\cite{sun2025tinyr1}.


While prior work reports promising gains on downstream benchmarks, a fundamental question remains underexplored: \textbf{how stable is the generation behavior of merged models over long horizons?}
In practice, many classical merging methods, including Task Arithmetic, TIES, DARE, and SCE, exhibit severe generation instabilities even when base models have near-zero repetition. As shown in \cref{fig:all_repetitions}, these methods can drastically amplify repetition, often approaching 100\%  across model scales. Qualitatively, this instability manifests as degenerate repetition loops and semantically incoherent outputs, as illustrated in \cref{fig:motivating_cases}.
Such instabilities are especially damaging for long-form generation and multi-step reasoning, where small distributional distortions compound over time.
Together, these observations expose a critical challenge: \textbf{can model merging integrate complementary capabilities while preserving long-horizon generation stability?}

\textbf{Our Approach}. We argue that generation failures in existing merging methods stem from a fundamental limitation: dense parameter-space fusion distorts the base model's output distribution, causing entropy collapse and self-reinforcing token selection. We introduce Sparse Complementary Fusion with reverse KL (SCF-RKL), a distribution-aware framework that selectively integrates complementary parameters using reverse Kullback-Leibler divergence\cite{bishop2006pattern} while enforcing sparsity to preserve generation stability. We evaluate SCF-RKL on 24 benchmarks across model scales (7B–32B) and architectures (Qwen2.5, LLaMA3, Mistral). Results show SCF-RKL consistently enhances reasoning and generalization, strongly suppresses degenerative repetition, maintains strong-model performance in asymmetric fusion, simultaneously improves reasoning and safety when integrating specialized models, and extends effectively to vision tasks, where it exhibits robust cross-modal generalization.

\textbf{Contributions}. First, we identify how existing methods amplify low-probability repetition behaviors, causing severe degradation. Second, we propose SCF-RKL, a principled sparse fusion framework using reverse KL divergence to preserve high-confidence components while avoiding repetition amplification. Third, we provide theoretical analysis showing reverse-KL-guided sparse fusion suppresses mode collapse with stronger stability guarantees than parameter averaging. Fourth, we demonstrate consistent empirical gains across 24 benchmarks in multiple fusion scenarios. Finally, we show SCF-RKL preserves model capacity ceilings and often surpasses both the strongest base model and existing baselines.

\section{Sparse Complementary Fusion With Reverse KL}
We propose Sparse Complementary Fusion with Reverse KL (SCF-RKL), a model merging framework that uses reverse KL divergence to selectively integrate complementary expertise from domain-specialized models. By formulating fusion as sparse updates, SCF-RKL preserves base-model coherence while substantially reducing degeneration such as repetition and garbled outputs.
\subsection{Problem Formulation}

Let $\theta_0 \in \mathbb{R}^{d}$ denote the parameters of a base model and  $\theta_1 \in \mathbb{R}^{d}$
 those of a complementary model trained or specialized under different objectives (e.g., reasoning, instruction-following, or safety alignment).
Our goal is to construct a fused model  $\theta_f$ that selectively incorporates the strengths of $\theta_1$ while maintaining the stable generative behavior of $\theta_0$.

We define the fused parameters as: 
\begin{align*}
\theta_f = \theta_0 + M \odot (\theta_1 - \theta_0)
\end{align*}
where  $M \in \{0, 1\}^d$ is a sparse fusion mask that determines which parameters are updated.

The \textbf{central challenge} is to design $M \in \{0, 1\}^d$ such that: (1) Only \textbf{informationally meaningful} updates are injected; (2) The fusion avoids large-scale parameter drift that leads to\textbf{ mode collapse or repetitive generation}.

\subsection{Importance Estimation via Reverse KL Divergence}
\noindent \textbf{Parameter Distributions} \quad
For each parameter group (e.g., per output channel or hidden dimension), we interpret parameters as defining a categorical distribution via softmax:
\begin{align*}
p = softmax(\theta_1), \qquad  q = softmax(\theta_0)
\end{align*}
where $p$ and $q$ represent the induced probability distributions of the complementary and base models, respectively.

\noindent \textbf{Reverse KL as an Information-Sensitive Metric} \quad
We measure the informational discrepancy between the two models using the \textbf{reverse Kullback–Leibler divergence}:
\begin{align*}
\text{RKL}(q \parallel p) = \sum_{i} q_i \log \frac{q_i}{p_i}
\label{eq:importance}
\end{align*}
Unlike the forward KL divergence, which emphasizes covering the modes of $p$, reverse KL penalizes deviations that significantly alter regions of high probability under the base model.
This property makes reverse KL particularly suitable for conservative fusion, as it prioritizes preserving the base model’s dominant probability mass.

\noindent \textbf{Information-theoretic Saliency Measure} \quad
Guided by the reverse KL divergence, we compute an information-theoretic saliency score for parameter selection:
\begin{align*}
I = |\theta_1 - \theta_0| \cdot \text{RKL}(q \parallel p) 
\end{align*}
This formulation couples \textbf{parameter deviation magnitude} with \textbf{information-geometric impact}, ensuring that large but informationally insignificant updates are suppressed.

\subsection{Dynamic Sparse Complementary Fusion}

To determine which parameters should be updated, we adopt a \textbf{distribution-adaptive thresholding strategy}.

Let $I$ denote the set of all information-theoretic saliency score. We compute the median $\text{med}(I)$ and the interquartile range $\text{IQR}(I) = Q_{high}(I) - Q_{low}(I)$. 
The fusion threshold is defined as:
\begin{align*}
\tau = \text{med}(I) + \alpha \cdot \text{IQR}(I)
\end{align*}
where $\alpha$ is a scale factor of IQR. The fusion mask is then:
\begin{align*}
    M_i = \mathbb{I}[I_i \geq \tau]
\end{align*}
This criterion selects parameters whose information-theoretic saliency scores lie in the upper statistical tail, yielding a sparse update that adapts automatically across layers, model scales, and architectures.

\begin{algorithm}[htb]
  \caption{Sparse Complementary Fusion with Reverse-
KL Sparsification}
  \label{alg:scf_fusion}
  \begin{algorithmic}
    \STATE {\bfseries Input:} base parameters $\theta_b$, secondary parameters $\theta_s$, 
    smoothing constant $\epsilon = 10^{-8}$, \\
    \quad lower quantile $q_{\text{low}} = 0.25$, upper quantile $q_{\text{high}} = 0.75$, \\
    \quad center quantile $q_{\text{center}} = 0.5$, IQR scale $\alpha = 1.5$
    \STATE {\bfseries Output:} fused parameters $\theta_{\text{fused}}$

    \STATE Compute absolute difference: $\Delta \gets |\theta_s - \theta_b|$
    \STATE Compute softmax distributions:
    \STATE \quad $q \gets \mathrm{softmax}(\theta_b) + \epsilon$
    \STATE \quad $p \gets \mathrm{softmax}(\theta_s) + \epsilon$
    
    \STATE Compute reverse KL divergence per output dimension:
    \STATE \quad $D_{\mathrm{KL}}^{\mathrm{rev}} \gets \sum_i q_i \log(q_i / p_i)$
    
    \STATE Broadcast $D_{\mathrm{KL}}^{\mathrm{rev}}$ to match shape of $\Delta$
    \STATE Compute importance scores: $\mathcal{I} \gets \Delta \odot D_{\mathrm{KL}}^{\mathrm{rev}}$
    
    \STATE Flatten $\mathcal{I}$ into a 1D tensor
    \STATE Compute quantiles:
    \STATE \quad $\text{Q1} \gets \mathrm{quantile}(\mathcal{I}, q_{\text{low}})$
    \STATE \quad $\text{Q3} \gets \mathrm{quantile}(\mathcal{I}, q_{\text{high}})$
    \STATE \quad $\mu \gets \mathrm{quantile}(\mathcal{I}, q_{\text{center}})$
    
    \STATE Set dynamic threshold: $\tau \gets \mu + \alpha \cdot (\text{Q3} - \text{Q1})$
    
    \STATE Construct binary mask: $M \gets \mathbf{1}[\mathcal{I} \geq \tau]$
    
    \STATE Fuse parameters: $\theta_{\text{fused}} \gets \theta_b + M \odot (\theta_s - \theta_b)$
    
    \STATE \textbf{return} $\theta_{\text{fused}}$
  \end{algorithmic}
\end{algorithm}

\section{Mathematical Analysis}

\subsection{Theoretical Guarantees Against Repetitive Generation}
\label{subsec:theory_repetition}

Repetitive or degenerate generation in merged models often stems from excessive perturbation of decoder parameters, which distorts token probability distributions and collapses entropy. Unlike dense fusion methods that modify all parameters, \textbf{SCF-RKL} inherently constrains such degeneration by strictly limiting updates to the most semantically significant directions.

We formalize this stability through two key theoretical results. First, we show that the divergence of the fused model is naturally bounded.

\begin{theorem}[Semantic Stability]
\label{thm:semantic_stability}
Let $q$ and $p$ denote the output distributions of the base and secondary models, respectively. If the fusion mask $M$ is constructed based on the Reverse-KL importance, the KL divergence of the fused distribution $q_f$ satisfies:
\[
D_{\mathrm{KL}}(q \,\|\, q_f) \leq D_{\mathrm{KL}}(q \,\|\, p).
\]
\end{theorem}
\textit{Proof Sketch.} The proof relies on the convexity of the Reverse-KL divergence. By construction, SCF-RKL only updates coordinates where the directional derivative of the divergence is maximized, ensuring $q_f$ remains in a subspace that minimizes deviation from the base model $q$. (See Appendix~\ref{app:proofs_repetition} for the full proof).

Second, we address the issue of entropy collapse (repetitive loops).

\begin{theorem}[Entropy Preservation]
\label{thm:entropy_preservation}
Under standard Lipschitz continuity assumptions for the softmax and entropy functions, the entropy of the fused model $H(q_f)$ is bounded below relative to the base model $H(q)$:
\[
H(q_f) \geq H(q) - L \cdot \big\| M \odot (\theta_s - \theta_b) \big\|_2,
\]
where $L$ is a Lipschitz constant.
\end{theorem}

Crucially, since $\| M \odot (\theta_s - \theta_b) \|_2 \ll \| \theta_s - \theta_b \|_2$ due to sparsity, the entropy drop in SCF-RKL is theoretically negligible compared to dense fusion methods. This ensures the merged model retains the generative diversity of the base model. Detailed proofs are provided in Appendix~\ref{app:proofs_repetition}.

\begin{figure}[tbp]
    \centering
    \begin{subfigure}[b]{0.49\linewidth}
        \centering
        \includegraphics[width=\linewidth]{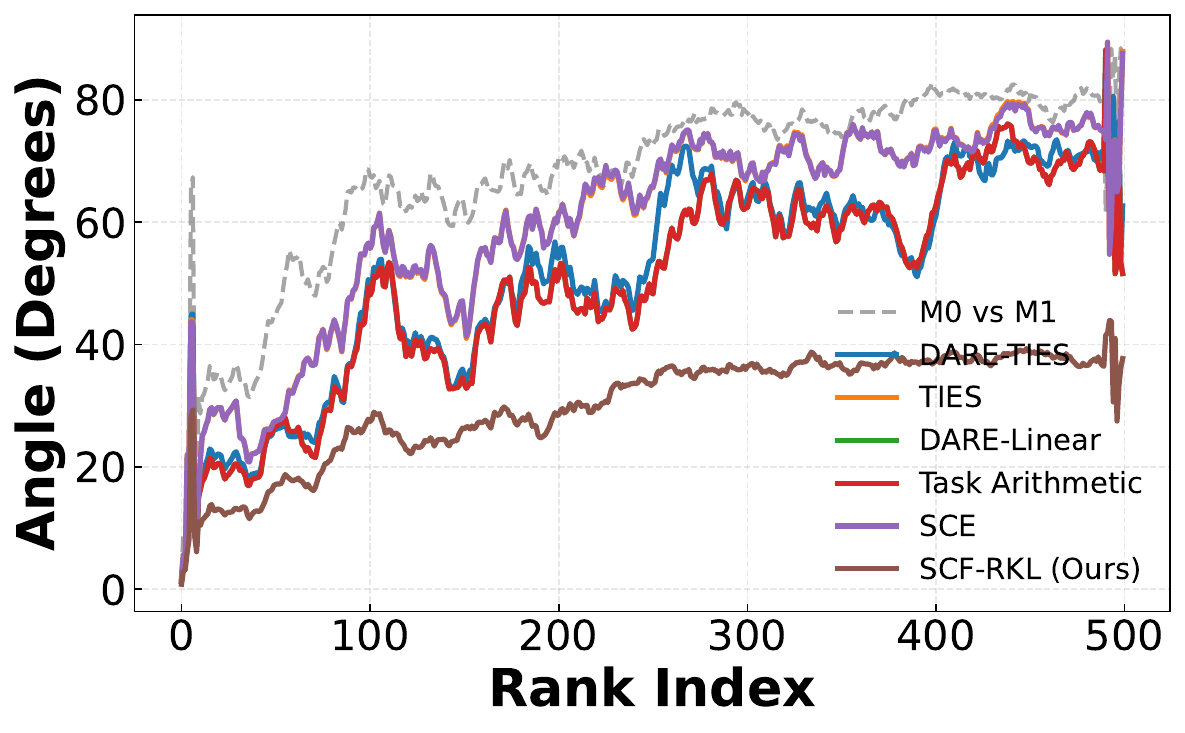}
        \caption{Rotation angles vs M0}
        \label{fig:rotation_m0}
    \end{subfigure}
    \hspace{-0.1cm}
    \begin{subfigure}[b]{0.49\linewidth}
        \centering
        \includegraphics[width=\linewidth]{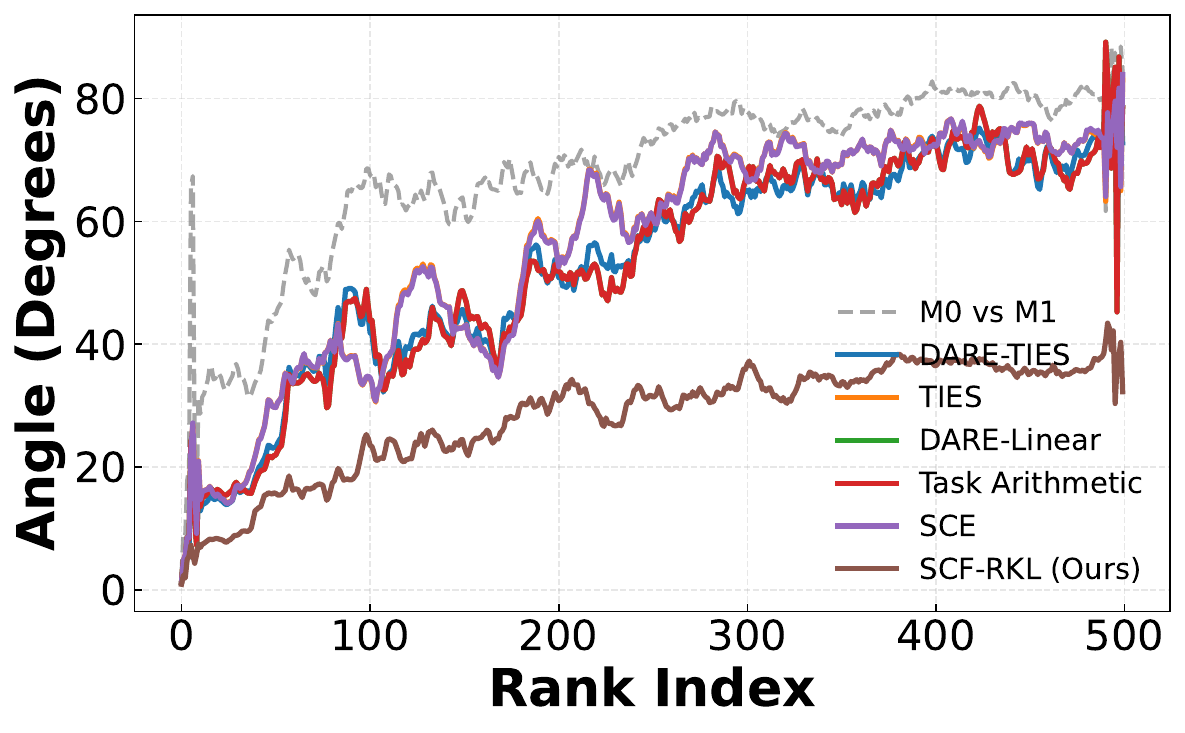}
        \caption{Rotation angles vs M1}
        \label{fig:rotation_m1}
    \end{subfigure}
    \caption{Principal angle rotation analysis at layer 15 across different fusion methods. The gray dashed line represents the baseline gap between parent models M0 and M1.}
    \label{fig:rotation_analysis}
    \vspace{-0.5cm}
\end{figure}

\begin{figure*}[tbp]
    \centering
    \includegraphics[width=0.35\textwidth]{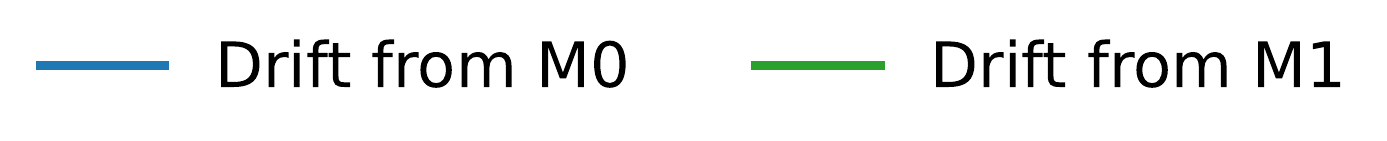}
    \vspace{-0.2cm}
    
    \begin{subfigure}[b]{0.16\textwidth}
        \centering
        \includegraphics[width=\textwidth]{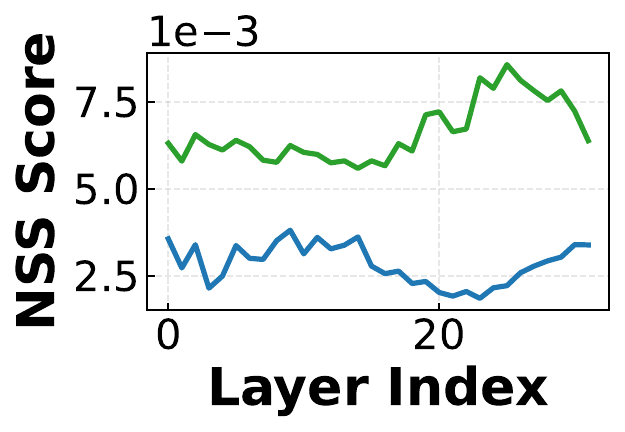}
        \caption{DARE-TIES}
        \label{fig:nss_dare_ties}
    \end{subfigure}
    \hfill
    \begin{subfigure}[b]{0.16\textwidth}
        \centering
        \includegraphics[width=\textwidth]{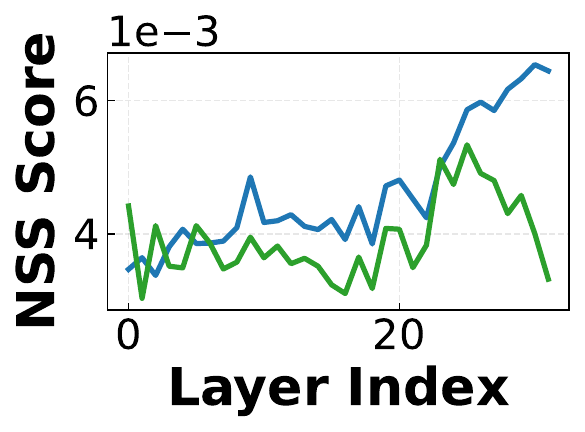}
        \caption{TIES}
        \label{fig:nss_ties}
    \end{subfigure}
    \hfill
    \begin{subfigure}[b]{0.16\textwidth}
        \centering
        \includegraphics[width=\textwidth]{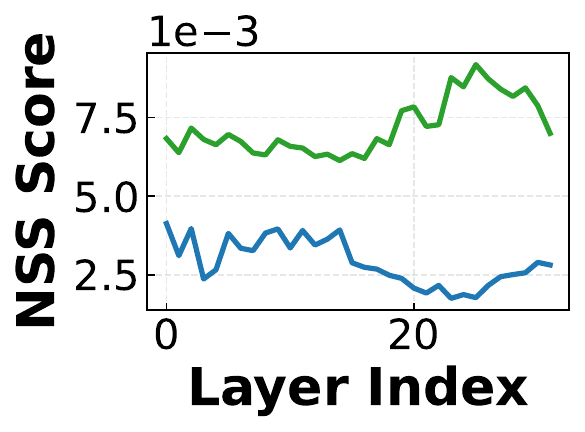}
        \caption{DARE-Linear}
        \label{fig:nss_dare_linear}
    \end{subfigure}
    \hfill
    \begin{subfigure}[b]{0.16\textwidth}
        \centering
        \includegraphics[width=\textwidth]{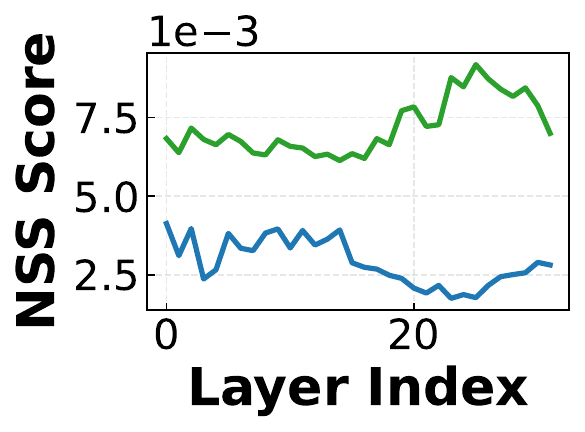}
        \caption{Task Arithmetic}
        \label{fig:nss_task_arithmetic}
    \end{subfigure}
    \hfill
    \begin{subfigure}[b]{0.16\textwidth}
        \centering
        \includegraphics[width=\textwidth]{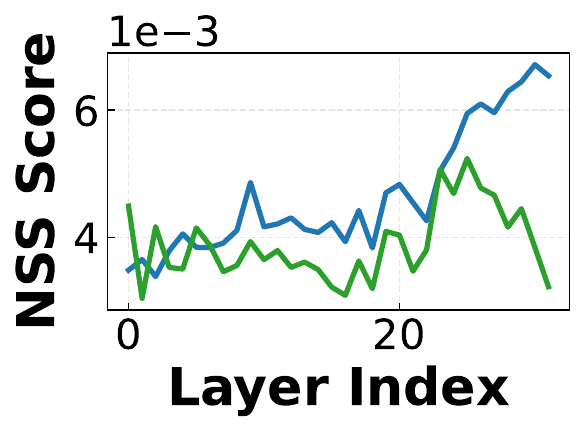}
        \caption{SCE}
        \label{fig:nss_sce}
    \end{subfigure}
    \hfill
    \begin{subfigure}[b]{0.16\textwidth}
        \centering
        \includegraphics[width=\textwidth]{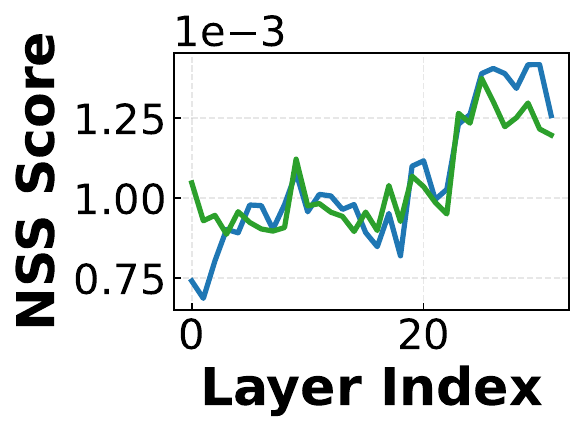}
        \caption{SCF-RKL (Ours)}
        \label{fig:nss_arcee}
    \end{subfigure}
    
    \caption{Normalized Spectral Shift (NSS) for all fusion methods. Lower NSS values indicate better preservation of the parent models' spectral properties. Note: The y-axis scale for SCF-RKL has been adjusted for better visualization.}
    \label{fig:nss_all}
    \vspace{-0.5cm}
\end{figure*}

\subsection{Geometric Consistency and Spectral Preservation}
\label{subsec:geometric_analysis}

While Theorems \ref{thm:semantic_stability} and \ref{thm:entropy_preservation} establish semantic stability, the mechanism of SCF-RKL can be further elucidated through spectral geometry. Recent findings in reinforcement learning with verifiable rewards (RLVR) suggest that effective optimization often occurs "off the principals", modifying parameters in low-curvature directions while preserving the pretrained eigenstructure \cite{zhu2025path}. We visualize these dynamics in Figures \ref{fig:rotation_analysis}--\ref{fig:nss_all}.

\textbf{Subspace Alignment and Principal Rotation.} 
Model merging effectively adds a perturbation matrix $E$ to the base weights. A critical risk in dense fusion (e.g., Task Arithmetic, TIES) is \textit{constructive interference}, where the superposition of dense updates fundamentally alters the principal subspaces. 
Figure \ref{fig:rotation_analysis} shows that standard methods induce significant rotation (approaching $80^\circ$) at Layer 15, implying the merged feature space drifts orthogonal to both parents. In contrast, SCF-RKL maintains a bounded rotation strictly below $40^\circ$. This consistency is further corroborated globally in Figure \ref{fig:max_angle_analysis}, where SCF-RKL achieves the lowest maximum principal angle deviation across all layers, confirming that our method interpolates the subspace geometry without the catastrophic rotation observed in dense techniques.

\begin{figure}[t]
    \centering
    \includegraphics[width=0.8\linewidth]{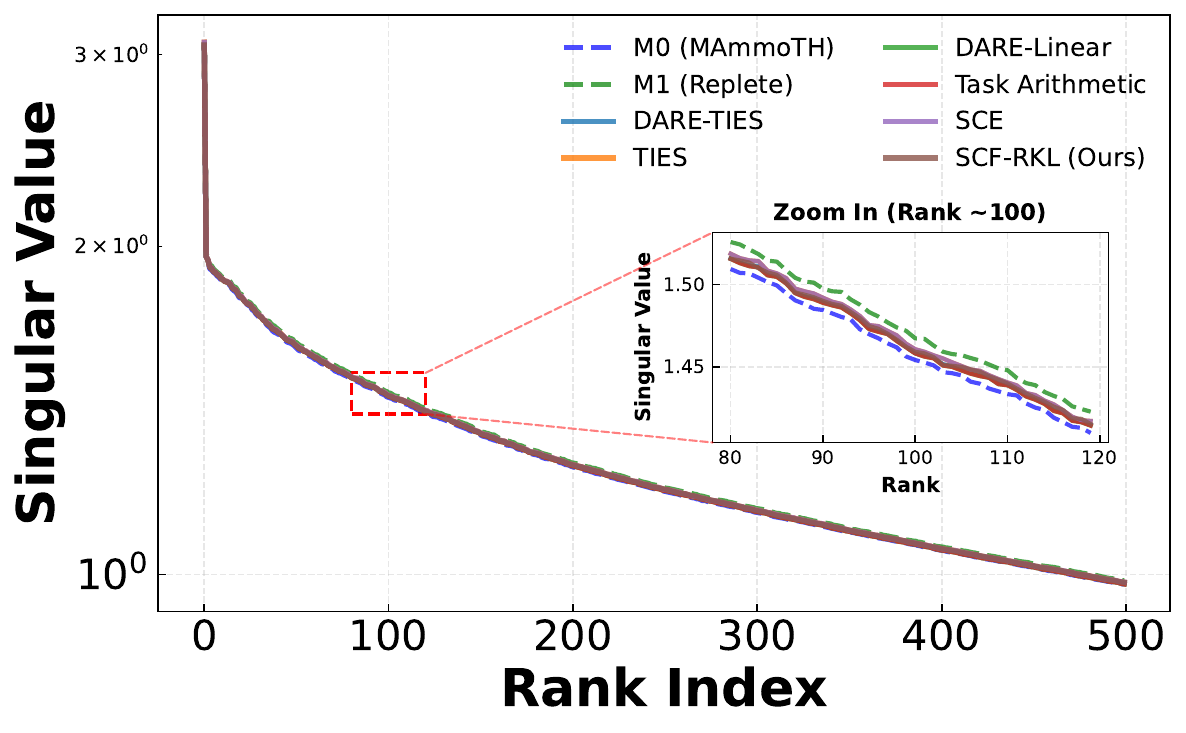}
    \caption{Singular value spectrum comparison. The inset shows a zoomed view around rank 100, revealing the subtle differences between fusion methods despite apparent overlap in the main plot.}
    \label{fig:spectrum}
\end{figure}

\begin{figure}[tbp]
    \centering
    \begin{subfigure}[b]{0.49\linewidth}
        \centering
        \includegraphics[width=\linewidth]{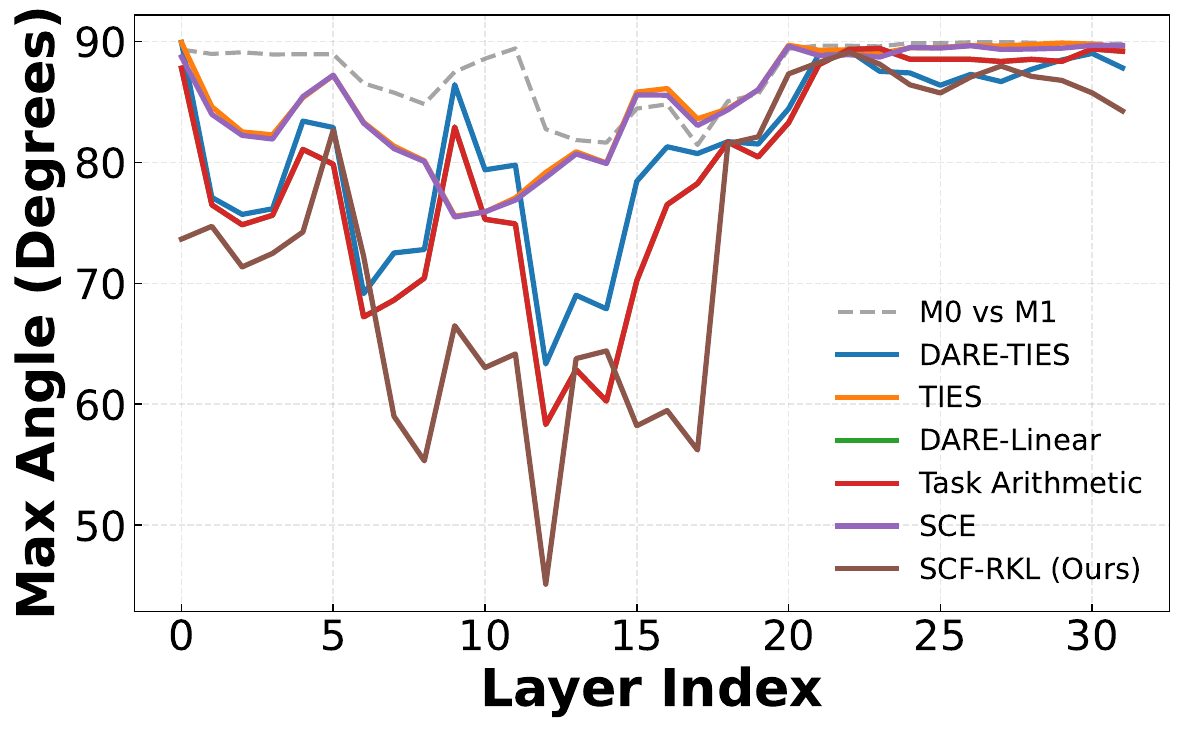}
        \caption{Max angle deviation vs M0}
        \label{fig:max_angle_m0}
    \end{subfigure}
    \hfill
    \begin{subfigure}[b]{0.49\linewidth}
        \centering
        \includegraphics[width=\linewidth]{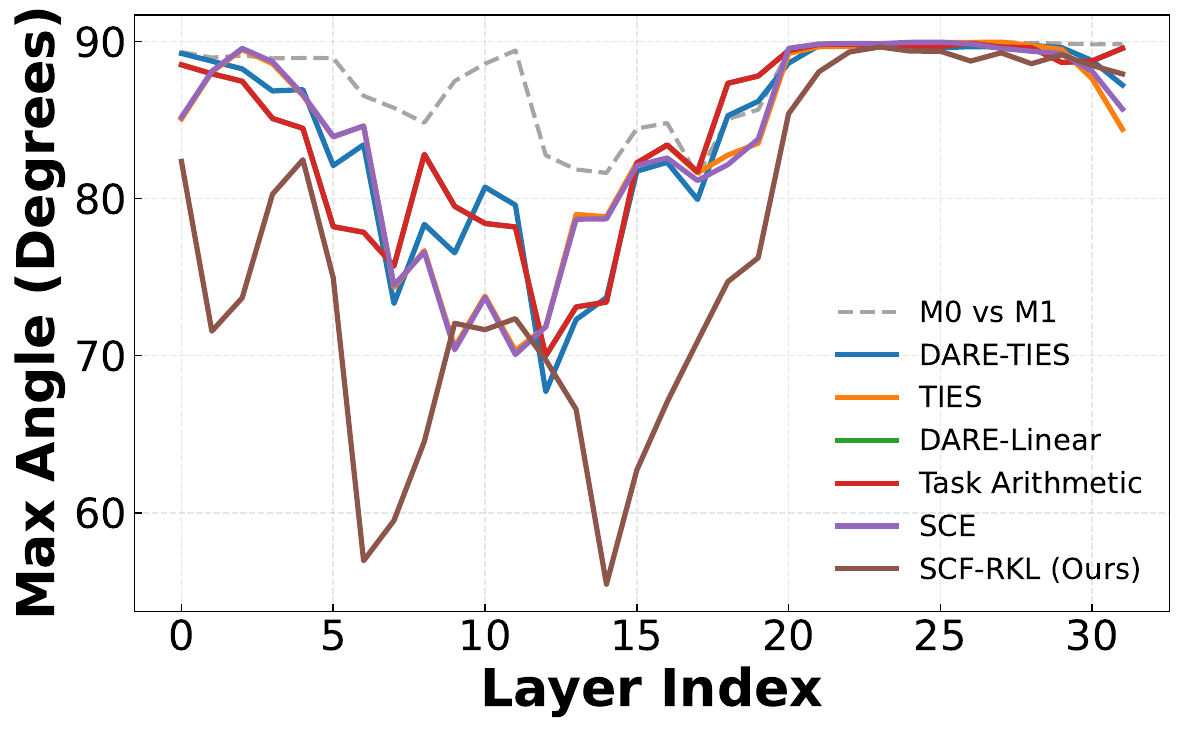}
        \caption{Max angle deviation vs M1}
        \label{fig:max_angle_m1}
    \end{subfigure}
    \caption{Maximum principal angle across all layers for different fusion methods. Lower values indicate better geometric consistency with the parent models.}
    \label{fig:max_angle_analysis}
    \vspace{-0.5cm}
\end{figure}

This geometric consistency is not accidental; it is a theoretical consequence of our sparsity constraint acting as a spectral filter.

\begin{theorem}[Bound on Subspace Rotation]
\label{thm:subspace_rotation}
Let $U$ and $\tilde{U}$ be the singular subspaces of the base and fused models, respectively. By Wedin's $\sin \Theta$ Theorem, the rotation of the principal angles is strictly bounded by the ratio of the perturbation norm to the spectral gap $\delta$:
\[
\sin \Theta_{\max}(U, \tilde{U}) \leq \frac{\| M \odot (\theta_s - \theta_b) \|_2}{\delta}.
\]
\end{theorem}
\textit{Proof Sketch.} The bound follows from matrix perturbation theory. Since SCF-RKL filters out noise-dominated updates, the numerator (perturbation norm) is significantly smaller than in dense methods ($\|E_{\text{sparse}}\| \ll \|E_{\text{dense}}\|$), guaranteeing minimal subspace rotation. (See Appendix~\ref{app:spectral_proofs} for full derivation).

\textbf{Spectral Energy Stability.} 
The singular value spectrum (Figure \ref{fig:spectrum}) dictates signal propagation. While baselines often overshoot the spectral envelope of the parents (zoomed inset), SCF-RKL strictly bounds the singular values between $M_0$ and $M_1$, acting as a geometric mean.
This stability is quantified by the Normalized Spectral Shift (NSS) in Figure \ref{fig:nss_all}. SCF-RKL achieves an NSS of $\approx 10^{-3}$, nearly an order of magnitude lower than baselines. As a corollary to Theorem \ref{thm:subspace_rotation}, minimizing the Frobenius norm of the update via sparse masking directly minimizes the upper bound of NSS, explaining the superior spectral preservation.

\begin{figure}[htbp]
    \centering
    \includegraphics[width=0.65\linewidth]{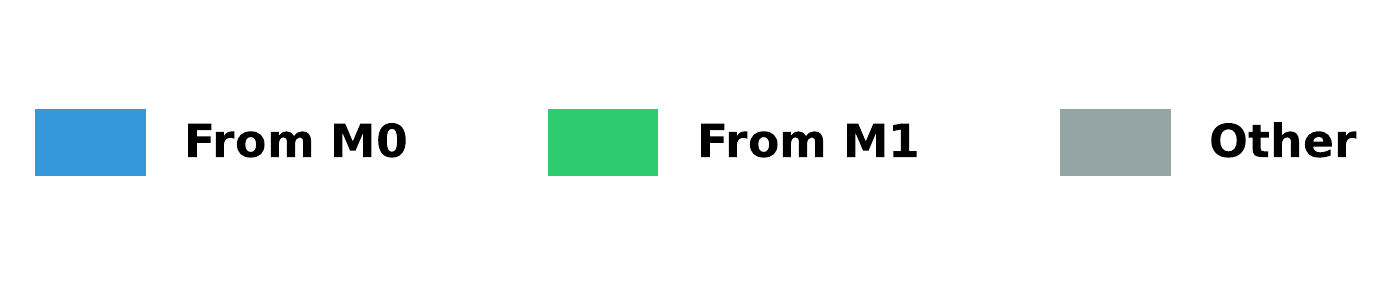}
    
    \includegraphics[width=0.32\linewidth]{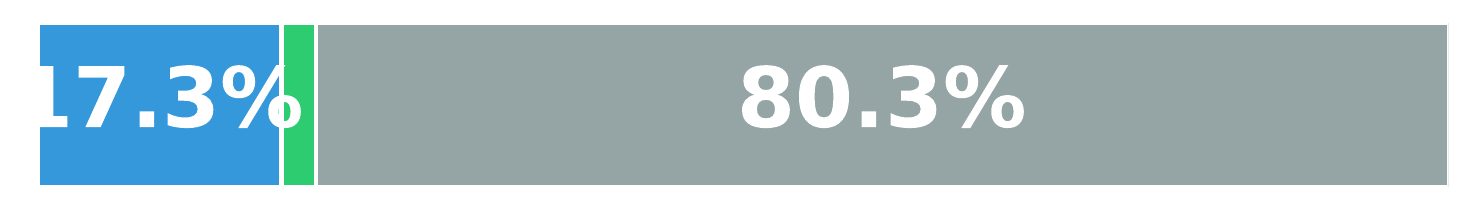}
    \hfill
    \includegraphics[width=0.32\linewidth]{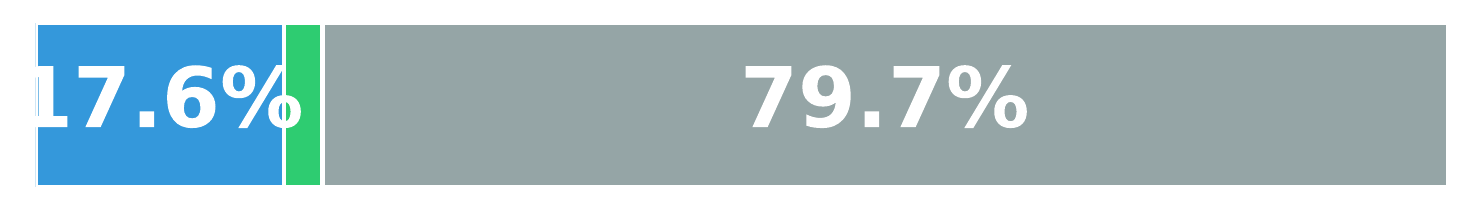}
    \hfill
    \includegraphics[width=0.32\linewidth]{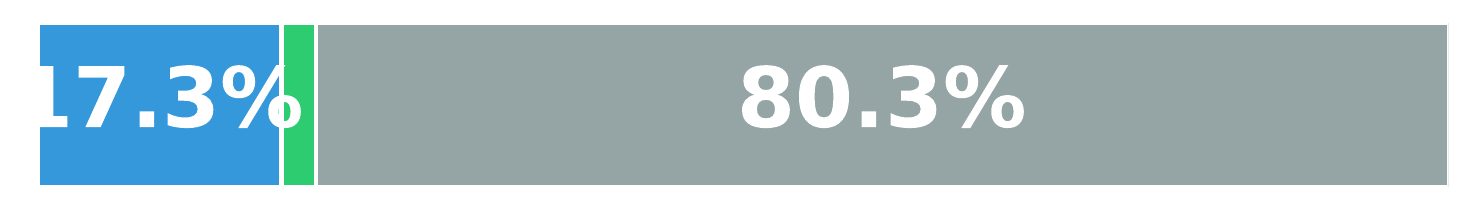} 

    \includegraphics[width=0.32\linewidth]{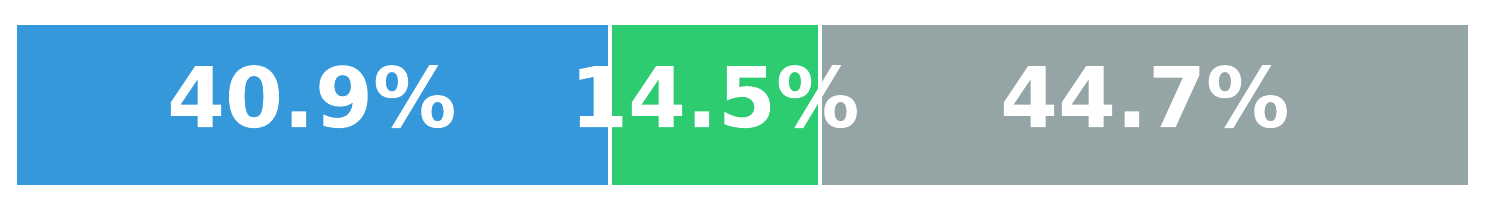}
    \hfill
    \includegraphics[width=0.32\linewidth]{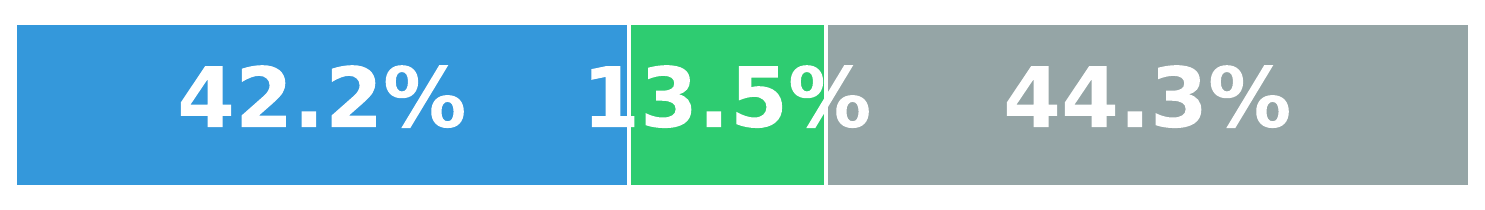}
    \hfill
    \includegraphics[width=0.32\linewidth]{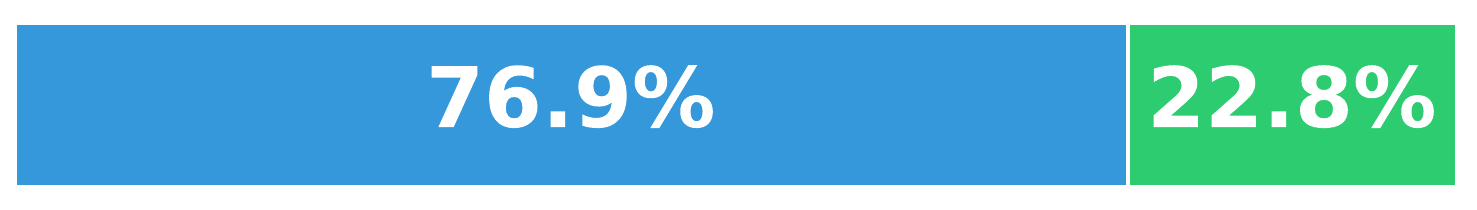}
    \caption{Parameter source distribution: (from top to bottom, left to right) DARE-Linear, DARE-TIES, Task Arithmetic, TIES, SCE, SCF-RKL (Ours).}
    \label{fig:param_distribution}
    \vspace{-0.1cm}
\end{figure}

\subsection{Discrete Composition and Manifold Adherence}
\label{subsec:discrete_composition}

A fundamental distinction between SCF-RKL and dense fusion lies in the provenance of the fused parameters. As illustrated in Figure \ref{fig:param_distribution}, baseline methods (e.g., Task Arithmetic, TIES) generate models where over 80\% of parameters originate from "Neither" parent. These values are linear interpolations: floating-point artifacts that were never seen during the training of either $M_0$ or $M_1$. Such parameters force the model into the unverified linear subspace between the parent manifolds.

\textbf{Discrete Composition Property.} 
In contrast, SCF-RKL adheres to a strict \textit{Discrete Composition} property. Since our mask $M$ is binary, every parameter in the fused model is explicitly sourced from either the base or the secondary model:
\[
\theta_{f, j} \in \{ \theta_{b, j}, \theta_{s, j} \}, \quad \forall j.
\]
This ensures the merged model remains on the union of the optimized manifolds ($M_0 \cup M_1$), avoiding the "phantom" energy created by averaging non-aligned vector spaces.

\begin{table*}[t]
\centering
\caption{Performance of merging TinyR1-Preview-Math-32B (Math), and TinyR1-Preview-Science-32B (Science) on all datasets. All metrics are reported as $\text{pass}@1$ scores.The best overall result (including base models) is \underline{underlined}; the best among fused models is \textbf{bolded}. 
}
\label{tab:tiny-r1-preview-32b-merging}
\scriptsize
\setlength{\tabcolsep}{4pt} 
\begin{tabularx}{\textwidth}{@{}l l
>{\centering\arraybackslash}X
>{\centering\arraybackslash}X
>{\centering\arraybackslash}X
>{\centering\arraybackslash}X
>{\centering\arraybackslash}X
>{\centering\arraybackslash}X
>{\centering\arraybackslash}X
>{\centering\arraybackslash}c
>{\centering\arraybackslash}X
>{\centering\arraybackslash}X
>{\centering\arraybackslash}X
@{}}
\toprule
\multirow{2}{*}{Merging Methods} 
& \multirow{2}{*}{Models} 
& \multicolumn{4}{c}{Advanced Reasoning} 
& \multicolumn{3}{c}{General Knowledge \& Reasoning} 
& \multicolumn{1}{c}{Code Gen} 
& \multicolumn{2}{c}{Instruction Following}
& \multirow{2}{*}{Avg} \\
\cmidrule(lr){3-6}
\cmidrule(lr){7-9}
\cmidrule(lr){10-10}
\cmidrule(lr){11-12}
& 
& \makecell{AIME24\\$\uparrow$} 
& \makecell{AIME25\\$\uparrow$} 
& \makecell{GPQA\\$\uparrow$} 
& \makecell{LCB\\$\uparrow$} 
& \makecell{GSM8K\\$\uparrow$} 
& \makecell{BBH\\$\uparrow$} 
& \makecell{MMLU\\$\uparrow$} 
& \makecell{Human\\Eval$\uparrow$} 
& \makecell{IFEval\\Strict$\uparrow$} 
& \makecell{IFBench\\Strict$\uparrow$} \\
\midrule
\multirow{2}{*}{w/o Merging} 
& Math 
& 74.58 
& 62.08 
& 67.99
& 61.65 
& 95.42
& 48.18
& \underline{78.10}
& 85.52
& \underline{63.42} 
& 7.76 
& 64.47\\
& Science
& \underline{78.02}
& 61.88 
& 68.78
& 61.47
& \underline{96.02}
& 61.11
& 51.65
& 85.52
& 60.79 
& 10.15 
& 63.54\\
\midrule
Task Arithmetic & Fuse & \textbf{77.50} & 61.25 & 68.18 & 57.80 & 26.57 & 75.26 & 68.56 & 85.98 & 51.62 & 6.27 & 57.90 \\
Dare Task Arithmetic & Fuse & \textbf{77.50} & 61.25 & 68.18 & 55.91 & 26.19 & \underline{\textbf{75.29}} & 68.84 & \underline{\textbf{86.59}} & 50.23 & 8.06 & 57.80\\
Ties Merging & Fuse & 74.38 & 61.04 & 67.08 & 50.18 & 24.39 & 57.43 & 73.70 & 85.52 & 52.52 & 9.25 & 55.55\\
Dare Ties Merging & Fuse & 76.15 & 60.62 & 67.08 & 56.90 & 27.39 & 73.33 & 71.27 & 85.06 & 55.45 & 8.96 & 58.22\\
SCE & Fuse & 74.38 & 61.04 & 67.08 & 49.91 & 24.68 & 57.49 & \textbf{74.74} & 84.45 & 53.51 & 7.46 & 55.47\\
SCF-RKL (Ours) & Fuse & 76.98 & \underline{\textbf{62.50}} & \underline{\textbf{69.19}} & \underline{\textbf{62.46}} &\textbf{95.20} & {70.31} & 61.17 & 86.28 & \textbf{57.90} & \underline{\textbf{11.34}} & \underline{\textbf{65.33}}\\
\bottomrule
\end{tabularx}
\end{table*}

\begin{table*}[t]
\centering
\caption{Performance of merging  
TinyR1-Preview-Math-14B (Math), and TinyR1-Preview-Code-14B (Code) 
on all datasets. 
All metrics are reported as $\text{pass}@1$ scores.
The best overall result (including base models) is \underline{underlined}; the best among fused models is \textbf{bolded}. }
\label{tab:tiny-r1-preview-14b-merging}
\scriptsize
\setlength{\tabcolsep}{3pt} 
\begin{tabularx}{\textwidth}{@{}l l
>{\centering\arraybackslash}X
>{\centering\arraybackslash}X
>{\centering\arraybackslash}X
>{\centering\arraybackslash}X
>{\centering\arraybackslash}X
>{\centering\arraybackslash}X
>{\centering\arraybackslash}X
>{\centering\arraybackslash}c
>{\centering\arraybackslash}X
>{\centering\arraybackslash}X
>{\centering\arraybackslash}X
@{}}
\toprule
\multirow{2}{*}{Merging Methods} 
& \multirow{2}{*}{Models} 
& \multicolumn{4}{c}{Advanced Reasoning} 
& \multicolumn{3}{c}{General Knowledge \& Reasoning} 
& \multicolumn{1}{c}{Code Gen} 
& \multicolumn{2}{c}{Instruction Following} 
& \multirow{2}{*}{Avg} \\
\cmidrule(lr){3-6}
\cmidrule(lr){7-9}
\cmidrule(lr){10-10}
\cmidrule(lr){11-12}
& 
& \makecell{AIME24\\$\uparrow$} 
& \makecell{AIME25\\$\uparrow$} 
& \makecell{GPQA\\$\uparrow$} 
& \makecell{LCB\\$\uparrow$} 
& \makecell{GSM8K\\$\uparrow$} 
& \makecell{BBH\\$\uparrow$} 
& \makecell{MMLU\\$\uparrow$} 
& \makecell{Human\\Eval$\uparrow$} 
& \makecell{IFEval\\Strict$\uparrow$} 
& \makecell{IFBench\\Strict$\uparrow$} \\
\midrule
\multirow{2}{*}{w/o Merging} 
& Math 
& \underline{74.27}
& \underline{61.25}
& \underline{61.65}
& 55.91
& \underline{93.73}
& 48.15
& 61.83
& 70.27
& \underline{64.93} 
& 8.36
& 60.64\\
& Code
& 72.50
& 58.44
& 60.01
& 55.20
& 93.46
& 69.29
& \underline{68.39}
& 84.15
& 43.05 
& 5.07 
& 60.26\\
\midrule
Task Arithmetic & Fuse & 61.98 & 54.17 & 51.39 & 38.53 & 71.49 & 66.15 & \textbf{65.17} & 84.15 & 38.08 & 10.48 & 54.06\\
Dare Task Arithmetic & Fuse & 62.92 & 54.06 & 51.39 & 38.53 & 71.44 & 66.17 & 64.85 & 83.08 & 37.50 & 11.34 & 54.23\\
Ties Merging & Fuse & 63.23 & 52.81 &52.08 & 44.18 & 70.74 & 60.70 & 58.94 & 84.91 & 34.70 & 9.55 & 53.78\\
Dare Ties Merging & Fuse & 61.88 & 49.27 & 51.07 & 40.05 & 71.74 & 64.75 & 63.76 & \underline{\textbf{85.67}} & 37.25 & 10.48 & 53.79 \\
SCE & Fuse & 65.62 & 56.25 & 51.61 & 44.35 & 72.65 & 60.53 & 57.12 & \underline{\textbf{85.67}} & 34.47 & \underline{\textbf{12.24}} & 54.51\\
SCF-RKL (Ours)& Fuse & \textbf{73.96} & \textbf{58.02} & \textbf{56.36} & \underline{\textbf{62.46}} &\textbf{93.42} & \underline{\textbf{69.92}} & 63.47 & 83.23 & \textbf{42.12} & 7.46 & \underline{\textbf{61.44}}\\
\bottomrule
\end{tabularx}
\end{table*}

\textbf{Base Model Dominance and Norm Constraints.} 
Empirically, the merged model retains a high proportion of base parameters ($\rho_{base} \to 1$). In Appendix~\ref{app:discrete_analysis}, we provide a statistical bound (Proposition~\ref{prop:base_ratio}) modeling the importance scores as a heavy-tailed distribution, proving that standard IQR thresholds naturally isolate the sparse, high-information tail.

This sparsity provides the final piece of the geometric puzzle. For linear merging, the perturbation error scales with the full dimensionality $d$: $\|E_{linear}\|_2 \propto \sqrt{d}$. For SCF-RKL, the error is restricted to the active set of size $k \ll d$:
\[
\|E_{sparse}\|_2 \approx \sqrt{k} \cdot \mathbb{E}_{\text{top-k}}[|\theta_s - \theta_b|] \ll \|E_{linear}\|_2.
\]
This structural constraint explains why the "Rotation vs M0" curve (Figure \ref{fig:rotation_m0}) remains low: SCF-RKL operates via surgical coordinate switching rather than global averaging.


\section{Experimental Setup}
We conduct extensive experiments to evaluate the effectiveness and robustness of SCF-RKL across diverse model architectures, parameter scales, fusion baselines, and evaluation benchmarks.
Our experimental design aims to assess not only task performance but also generation stability and robustness, particularly under reasoning-intensive and safety-critical settings, extending to vision tasks where SCF-RKL demonstrates strong cross-modal transferability.


\noindent \textbf{Models} \quad Experiments are conducted on three open-source model families, Mistral, LLaMA 3, and Qwen2.5, spanning parameter scales from 7B to 32B. We evaluate Mistral-7B models by fusing Mistral-7B-Instruct-v0.2 with MathCoder2-Mistral-7B, and LLaMA-3-8B models by fusing Meta-Llama-3-8B-Instruct with  MAmmoTH2-8B-Plus. For Qwen2.5 at 14B and 32B scales, DeepSeek-R1-Distill-Qwen2.5 models are fused with TinyR1-Preview variants specialized for math, code, or science. These models are trained under heterogeneous objectives, enabling systematic evaluation of fusion robustness across tasks. We further validate SCF-RKL on vision tasks using ViT-B/32 models fine-tuned on seven image classification datasets.

\noindent \textbf{Fusion Baselines} \quad 
We compare SCF-RKL with widely adopted model merging methods, including Task Arithmetic\cite{ilharco2022editing}, DARE-Task Arithmetic\cite{ilharco2022editing}, TIES-Merging\cite{yadav2023ties}, DARE-TIES\cite{yadav2023ties}, and SCE\cite{wan2025fusechat}.
All fusion methods are applied under their standard configurations without additional fine-tuning, ensuring fair and controlled comparisons.

\noindent \textbf{Benchmarks} \quad 
Evaluation covers five benchmark categories: advanced reasoning, general knowledge and reasoning, code generation, instruction following, and safety.
Advanced reasoning benchmarks include AIME24, AIME25\cite{aime2025}, GPQA\cite{rein2024gpqa}, and LiveCodeBench v5\cite{jain2024livecodebench}; general reasoning benchmarks include GSM8K, BBH, and MMLU; code generation is evaluated on HumanEval.
Instruction-following ability is assessed using strict-prompt protocols on IFEval\cite{zhou2023instruction} and IFBench\cite{pyatkin2025generalizing}.
Safety evaluation spans base risk assessment, adversarial attacks, jailbreak attacks, and extreme safety stress tests, including S-Eval\cite{yuan2025s}, HarmBench\cite{mazeika2024harmbench}, JBB-Behaviors\cite{chao2024jailbreakbench}, WildJailbreak\cite{wildteaming2024}, and StrongREJECT suites. Vision classification is evaluated on seven standard datasets: SUN397\cite{xiao2016sun}, Cars\cite{krause20133d}, RESISC45\cite{cheng2017remote}, DTD\cite{cimpoi2014describing}, GTSRB\cite{stallkamp2011german}, MNIST\cite{lecun1998mnist}, and SVHN\cite{netzer2011reading}.


\noindent \textbf{Hyperparameters} \quad For SCF-RKL, we use default hyperparameters following Tukey's rule: $\epsilon = 10^{-8}$, $q_{\text{low}} = 0.25$, $q_{\text{high}} = 0.75$, $q_{\text{center}} = 0.5$, $\alpha = 1.5$, yielding 5--15\% sparsity per layer.

\section{Results of Experiments}
\begin{figure}[h]
  \vskip 0.1in
  \begin{center}
    \centerline{\includegraphics[width=\columnwidth]{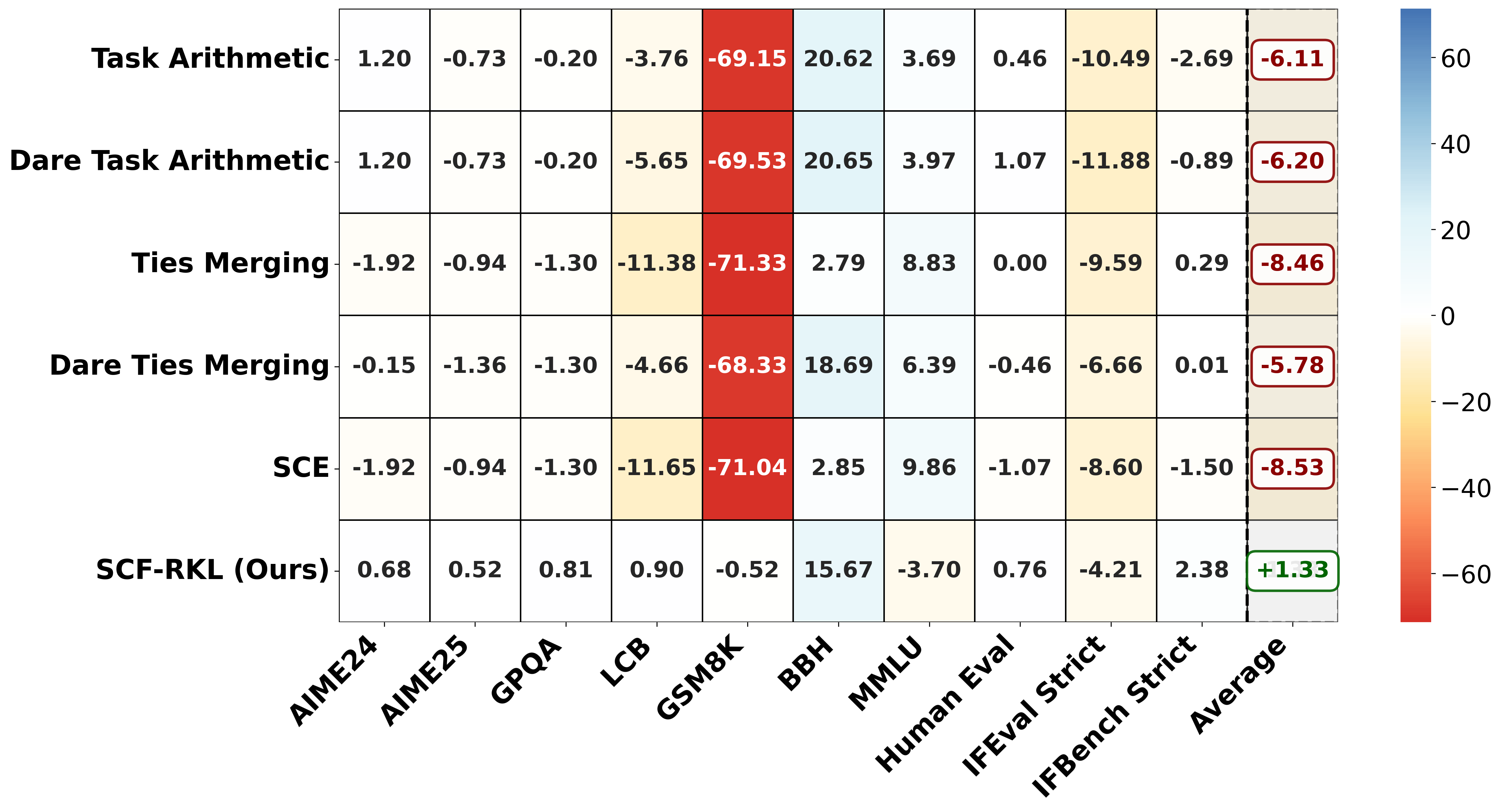}}
    \caption{
      Per-benchmark accuracy change of fused models at the 32B scale, measured as the difference between the fused model’s score and the average score of the two base models.Positive values (blue) indicate improvement; negative values (red) indicate degradation.
    }
    \label{fig: heatmpa accuracy_changes_after_merging-32b}
  \end{center}
\end{figure}
\setlength{\intextsep}{0.0pt}

We present experimental results demonstrating SCF-RKL's effectiveness on reasoning-centric and instruction-tuned models, emphasizing performance gains, robustness against repetition, and safety preservation.

\noindent \textbf{SCF-RKL integrates complementary 32B reasoning models, achieving consistent gains without degenerative repetition.} \quad
\cref{tab:tiny-r1-preview-32b-merging} shows TinyR1-Preview-Math-32B and TinyR1-Preview-Science-32B, which excel on different benchmarks: Math on AIME25, LiveCodeBench, MMLU; Science on AIME24, GPQA, BBH. After SCF-RKL fusion, the model preserves and improves both sets of capabilities, achieving \textbf{SOTA performance} on AIME25, GPQA, and LiveCodeBench, and strong results on GSM8K, IFEval, and IFBench, \textbf{outperforming all fusion approaches}.

Aggregated across Advanced Reasoning, General Knowledge \& Reasoning, Code Generation, and Instruction Following, SCF-RKL attains the highest average, surpassing all baselines and the strongest base model by 0.86 points.

\cref{fig: heatmpa accuracy_changes_after_merging-32b} shows per-task performance relative to the base models’ mean. SCF-RKL slightly decreases on GSM8K (-$0.52$), MMLU, and IFEval, but improves six of nine benchmarks, for a net gain of $+1.33$. Other fusion methods degrade by 5-9 points on average.

\cref{fig:all_repetitions} highlights severe repetition in prior fusion methods (near 100\%), versus 1\% for the base model and 0.2\% for SCF-RKL, explaining the prior performance collapse. SCF-RKL suppresses such degeneration while preserving reasoning accuracy.

\noindent \textbf{SCF-RKL integrates complementary 14B capabilities, delivering robust gains on reasoning and code benchmarks without inducing repetition.} \quad
\cref{tab:tiny-r1-preview-14b-merging} shows TinyR1-Preview-Math-14B and TinyR1-Preview-Code-14B, with complementary strengths: Math excels in mathematical/logical reasoning, Code in program synthesis and execution. SCF-RKL fusion preserves both capabilities and outperforms existing methods, achieving \textbf{SOTA on 7 of 10 benchmarks} and the highest overall average, surpassing all baselines and the strongest base model by 0.8 points.


Similar to the 32B case, prior fusion methods show severe repetition on GSM8K, approaching 100\% and causing $>20$-point degradation. SCF-RKL maintains only 0.2\% repetition, lower than the base models, avoiding catastrophic failure while preserving reasoning accuracy (\cref{fig:all_repetitions}).

\begin{table}[t]
\centering
\caption{Meta-Llama-3-8B-Instruct(base0), and MAmmoTH2-8B-Plus (base1)
on all datasets. 
All metrics are reported as $\text{pass}@1$ scores.
The best overall result (including base models) is \underline{underlined}; the best among fused models is \textbf{bolded}. }
\label{tab:llama3-8b-pass1}
\scriptsize
\setlength{\tabcolsep}{3pt}
\begin{tabularx}{\columnwidth}{@{}lcccccc@{}}
\toprule
Merging Methods 
& Models 
& \makecell{GSM8K\\$\uparrow$} 
& \makecell{BBH\\$\uparrow$} 
& \makecell{MMLU\\$\uparrow$} 
& \makecell{Human\\Eval$\uparrow$} 
& \makecell{Average} \\  
\midrule
\multirow{2}{*}{w/o Merging} 
& base0 
& \underline{73.35}
& \underline{65.78}
& 55.11
& 43.14 
& \underline{59.85}\\
& base1
& 56.54
& 53.92
& 38.68
& 24.70
& 43.46\\
\midrule
Task Arithmetic & Fuse  & 49.56 & 63.63 & 55.58 & 41.62 & 52.60 \\
Dare Task Arithmetic & Fuse  & 1.80 & 26.98 & 45.54 & 12.65 & 21.74 \\
Ties Merging & Fuse  & 6.31 & 57.89 & 52.04 & \underline{\textbf{44.51}} & 40.69 \\
Dare Ties Merging & Fuse  & 50.11 & \textbf{63.82} & \underline{\textbf{55.64}} & 38.11 & 51.92  \\
SCE & Fuse  & 1.72 & 24.48 & 42.85 & 14.18 & 20.31 \\
SCF-RKL (Ours)& Fuse &\textbf{71.68} & 51.64 & 46.17 & 42.53 & \textbf{53.01}\\
\bottomrule
\end{tabularx}
\end{table}

\begin{table}[t]
\centering
\caption{Performance of merging Mistral-7B-Instruct-v0.2 (base0), and MathCoder2-Mistral-7B (base1) 
on all datasets. 
All metrics are reported as $\text{pass}@1$ scores.
The best overall result (including base models) is \underline{underlined}; the best among fused models is \textbf{bolded}. }
\label{tab:mistral-7b}
\scriptsize
\setlength{\tabcolsep}{3pt} 
\begin{tabularx}{\columnwidth}{@{}lcccccc@{}}
\toprule
Merging Methods 
& Models 
& \makecell{GSM8K\\$\uparrow$}
& \makecell{BBH\\$\uparrow$}
& \makecell{MMLU\\$\uparrow$}
& \makecell{Human \\Eval$\uparrow$} 
& \makecell{Average}\\
\midrule
\multirow[c]{2}{*}{w/o Merging} 
& base0 & 46.76 & 47.38 & \underline{43.23} & 9.60  & 36.74\\
& base1 & 14.95 & 46.66 & 19.14 & 0.15  & 20.23\\
\midrule
Task Arithmetic & Fuse & 52.08 & 53.66 & 26.29 & \underline{\textbf{23.93}} & 38.99 \\
Dare Task Arithmetic & Fuse  & 52.08 & 53.66 & 26.29 & \underline{\textbf{23.93}} & 38.99 \\
Ties Merging & Fuse & 14.75 & 46.60 & 19.37 & 0.00 & 20.18 \\
SCE & Fuse  & 14.39 & 46.67 & 19.28 & 0.00 & 20.09 \\
Dare Ties Merging & Fuse & \underline{\textbf{58.53}} & \underline{\textbf{55.59}} & \textbf{41.78} & 19.97 & \underline{\textbf{43.97}} \\
SCF-RKL (Ours)& Fuse 
& 58.38 
& 54.47 
& 38.79 
& 23.17 
& 43.70\\
\bottomrule
\end{tabularx}
\end{table}

\noindent \textbf{SCF-RKL robustly preserves strong-model capabilities when merging with weaker counterparts.}
\cref{tab:llama3-8b-pass1}  and \cref{tab:mistral-7b} examine a more challenging non-complementary fusion setting, where a strong base model is merged with a substantially weaker counterpart.
In such scenarios, existing fusion methods often suffer from destructive interference, leading to severe performance degradation relative to the stronger model.

In contrast, SCF-RKL remains remarkably stable, producing fused models whose performance closely matches that of the strong base model and, in some cases, even surpasses it. As shown in \cref{tab:mistral-7b}, when merging Math and Science models based on Mistral-7B-Instruct, the SCF-RKL fused model outperforms the strongest base model by up to 7 points on General Knowledge \& Reasoning and Code Generation benchmarks.

Similar trends are observed on \cref{tab:llama3-8b-pass1} for LLaMA-3-8B-Instruct, where SCF-RKL consistently outperforms existing fusion baselines while maintaining performance comparable to the strongest base model. These results demonstrate that SCF-RKL not only excels in complementary fusion but also effectively avoids capability erosion in strong–weak merging scenarios, highlighting its robustness and practical applicability.

\textbf{Importantly, SCF-RKL does not compromise the model's capability ceiling, preserving Pass@n performance (see Appendix \cref{fig: bar line at passn 8b,fig: bar line passn 14b,fig: bar line passn 32b}).}

\begin{figure}[t]
  \vskip 0.2in
  \begin{center}
    \centerline{\includegraphics[width=\columnwidth]{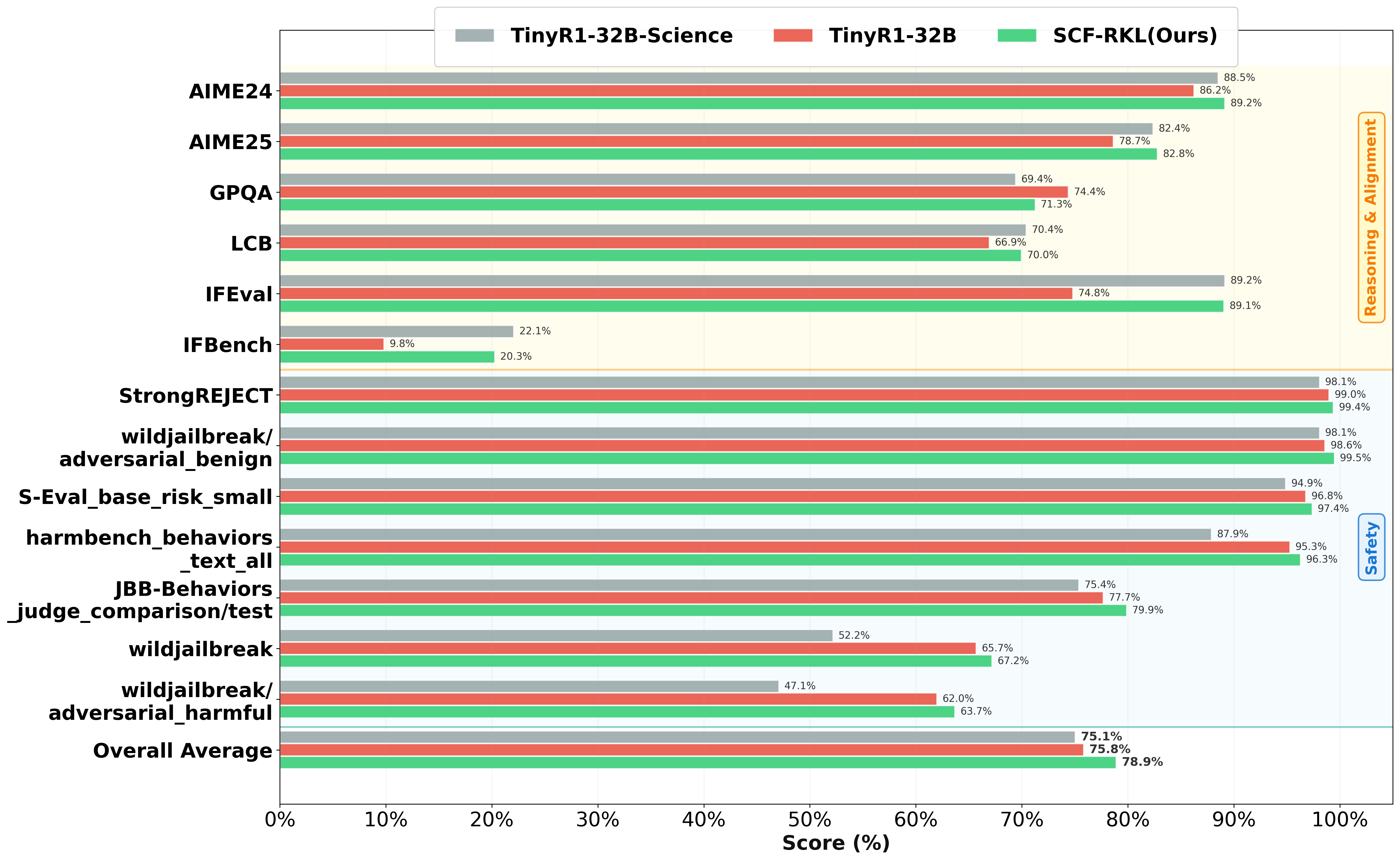}}
    \caption{ Comparison of two Deepseek-R1-Distill-Qwen2.5-32B-Derived models and the SCF-RKL fused model across 4 advanced reasoning, 2 instruction-following, and 7 safety benchmarks.
    }
    \label{fig: bar 32b comprehensive performance optimize}
  \end{center}
  \vspace{-0.5cm}
\end{figure}

\setlength{\intextsep}{0pt}
\begin{figure}[t]
  \vskip 0.1in
  \begin{center}
    \centerline{\includegraphics[width=0.85\columnwidth]{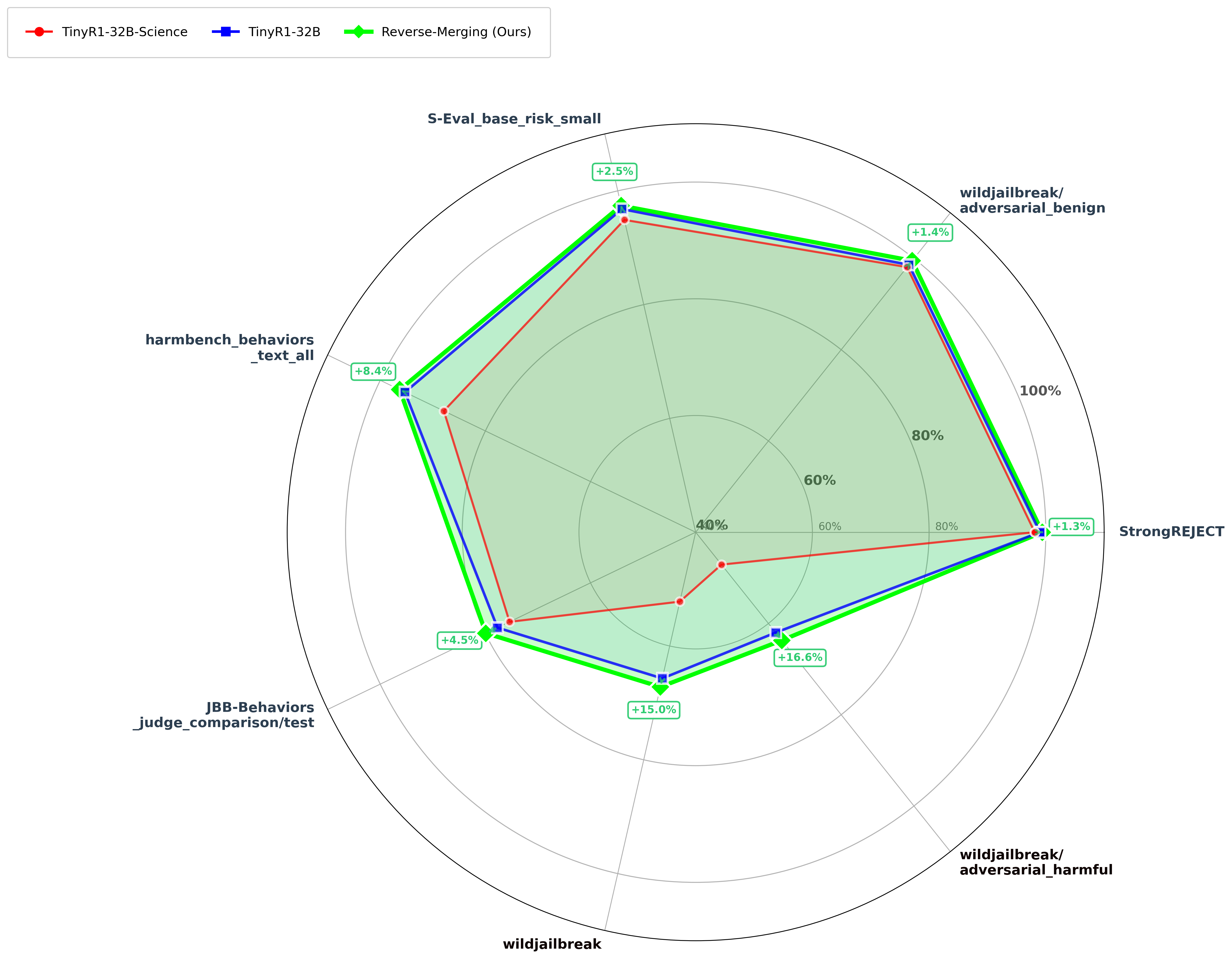}}
    \caption{Safety Benchmark Profiles of Base and Fused Models ( Deepseek-R1-Distill-Qwen2.5-32B-Derived models)
    }
    \label{fig:larda 32b safety}
  \end{center}
  \vspace{-0.5cm}
\end{figure}
\setlength{\intextsep}{0.0pt}

\begin{table}[h]
\centering
\vskip 0.2cm
\caption{Vision accuracy (\%). SCF-RKL surpasses the fine-tuned average on all 7 benchmarks and achieves the highest overall average among fusion methods. The best and second-best among fused models are \textbf{bolded} \underline{underlined}, respectively.}
\label{tab:vision_ultra_tight}
\scriptsize
\setlength{\tabcolsep}{3pt}
\begin{tabular}{@{}lcccccccc@{}}
\toprule
Method & Cars & SUN397 & RESISC45 & DTD & GTSRB & MNIST & SVHN & Avg \\
\midrule
FT Avg   & 49.0 & 54.1 & 48.1 & 39.7 & 35.8 & 60.5 & 43.1 & 47.2 \\
SCE    & 37.9  & 35.9  & 15.2  & 23.3  & 28.1  & 83.6  & 97.2  & 45.89 \\
TA     & \textbf{62.0} & \textbf{63.7} & \textbf{66.4} & \textbf{47.0} & \underline{52.5} & \underline{83.4} & \underline{61.0} & \underline{62.29} \\
SCF-RKL   & \underline{56.1} & \underline{59.8} & \underline{60.4} & \underline{40.9} & \textbf{53.9} & \textbf{89.6} & \textbf{86.9} & \textbf{63.94} \\
\bottomrule
\end{tabular}
\vskip 0.3cm
\end{table}
\setlength{\intextsep}{0.0pt}

\textbf{SCF-RKL simultaneously improves reasoning performance and safety robustness when fusing reasoning and safety-specialized models.}
It is widely observed that safety alignment and reasoning capability often exhibit a trade-off: enhancing one typically degrades the other. In contrast, as shown in \cref{fig: bar 32b comprehensive performance optimize} and \cref{fig:larda 32b safety}, SCF-RKL successfully fuses a safety-hardened model with a reasoning-specialized model to produce a single model that improves both dimensions simultaneously. 
Remarkably, SCF-RKL is able to improve both dimensions simultaneously. The fused model achieves consistent gains in reasoning performance while also exhibiting substantially enhanced safety robustness. Overall, as shown in \cref{fig: bar 32b comprehensive performance optimize} the SCF-RKL fused model \textbf{outperforms the strongest base model by approximately 3 points in terms of average performance across all benchmarks}.

Notably, as shown in \cref{fig:larda 32b safety}, the SCF-RKL fused model surpasses both base models on all seven safety benchmarks, demonstrating effective integration of safety-critical parameters while preserving or even improving reasoning and general capabilities. These results suggest that SCF-RKL mitigates the commonly observed safety–capability trade-off and offers a principled approach to building models that are both capable and robust.

\noindent \textbf{SCF-RKL generalizes effectively beyond language tasks, consistently enhancing performance in vision-domain model fusion.} \quad
We evaluate SCF-RKL on seven image classification benchmarks by fusing models fine-tuned on individual datasets. As shown in \cref{tab:vision_ultra_tight}, SCF-RKL surpasses the fine-tuned average (denoted \textbf{FT-Avg}; i.e., the mean performance of individually fine-tuned models on each of the 7 benchmarks) on all 7 tasks and \textbf{achieves the highest mean accuracy (63.9\%)} among all fusion methods, outperforming Task Arithmetic(\textbf{TA}) (62.3\%) and SCE (45.9\%). 
This demonstrates that our semantics-aware, sparse fusion mechanism transfers robustly to non-language modalities, with consistent task-specific performance.


\section{Related Work}

Model merging has emerged as a lightweight alternative to retraining for integrating multiple specialized models. Existing approaches can be broadly categorized into heuristic parameter merging, data-guided adaptive merging, and structure-aware fusion methods.

\textbf{Heuristic parameter merging methods}\cite{yadav2023ties,ilharco2022editing,wan2025fusechat} perform direct linear combinations of model parameters or task vectors, often augmented with sparsification or sign-consensus heuristics. Representative examples include Task Arithmetic\cite{ilharco2022editing}, TIES-Merging\cite{yadav2023ties}, DARE-based methods, and SCE\cite{wan2025fusechat}. While simple and efficient, these approaches typically rely on parameter magnitude or sign as proxies for importance, which can fail to capture the true behavioral impact of updates and frequently lead to destructive interference, degraded generalization, or degenerate generation behaviors.

\textbf{Data- or objective-guided merging methods} leverage additional signals to adaptively determine merging coefficients. AdaMerging\cite{yang2023adamerging} optimizes task-wise or layer-wise fusion weights using entropy minimization on unlabeled test data, while LED-Merging\cite{ma2025led} identifies task-critical neurons through gradient-based attribution to mitigate safety–utility conflicts. Although effective in specific settings, these methods require access to data or surrogate objectives, limiting their applicability in data-scarce or deployment-sensitive scenarios.

\textbf{Structure-aware and subspace-based approaches} focus on where capabilities reside within the model architecture. Methods such as PCB-Merging\cite{du2024parameterPCB-merging} balance parameter competition across tasks, and recent work on multimodal fusion analyzes how reasoning or perception capabilities localize to specific layers. These approaches introduce valuable architectural insights but often rely on structural assumptions and do not explicitly model the distributional impact of parameter updates.

In contrast to the above, our work introduces a distribution-aware model merging framework. SCF-RKL quantifies parameter importance via reverse KL divergence, measuring how updates perturb the base model’s output distribution, and applies adaptive sparse selection to preserve high-probability regions. This design enables SCF-RKL to remain entirely data-free while effectively mitigating destructive interference, degenerative repetition, and safety degradation, challenges that are not simultaneously addressed by prior approaches.

\section{Conclusion}
We identify a critical failure mode of existing task-specific model merging methods: even rare repetition behaviors in base models can be severely amplified after merging, leading to substantial performance degradation and wasted inference cost. We show that this issue arises systematically from insufficient distribution-aware importance modeling.

To address this, we propose SCF-RKL, a data-free merging framework based on sparse complementary fusion and reverse KL-based importance estimation. Our theoretical analysis shows that reverse KL preserves high-probability regions of the base model distribution, preventing the amplification of low-probability degenerative behaviors. Empirically, SCF-RKL consistently avoids repetition collapse observed in prior methods.

Across diverse architectures, scales, and benchmarks, SCF-RKL effectively integrates complementary capabilities, preserves strong-model performance in strong-weak fusion, and does not reduce the attainable performance upper bound. Notably, when merging reasoning- and safety-specialized models, SCF-RKL simultaneously improves both reasoning accuracy and safety robustness, mitigating the common trade-off between these objectives.
Moreover, SCF-RKL demonstrates consistent effectiveness beyond language tasks, extending to vision-domain model fusion across diverse benchmarks.


Overall, SCF-RKL provides a principled and robust solution for stable model merging, supported by both theoretical guarantees and extensive experimental validation. Future work will explore layer-adaptive importance estimation and theoretical analysis of multi-step generation stability.

\clearpage
\section*{Impact Statement}
This work focuses on improving the stability and robustness of model merging techniques for large language models. While model merging enables efficient reuse and integration of specialized models, improper fusion can lead to unstable generation behaviors, including severe degeneration and safety failures.

By explicitly addressing long-horizon generation stability and reducing the amplification of rare failure modes, our method contributes to safer and more reliable deployment of merged models. We do not anticipate direct negative societal impacts from this work. As with all advances in large-scale language modeling, responsible development and
deployment require rigorous safety evaluation and alignment with human values.

\nocite{langley00}

\bibliography{fuse-paper}
\bibliographystyle{icml2026}

\newpage
\appendix
\onecolumn

\section{Proofs for Mathematical Analysis}
\label{app:proofs}

\subsection{Proofs of Semantic Stability and Entropy Preservation}
\label{app:proofs_repetition}

In this section, we provide the rigorous proofs for the theorems presented in Section~\ref{subsec:theory_repetition}.

\paragraph{Restatement of Theorem~\ref{thm:semantic_stability} (Semantic Stability).}
\textit{Let $q = \mathrm{softmax}(\theta_b)$, $p = \mathrm{softmax}(\theta_s)$, and $q_f = \mathrm{softmax}(\theta_f)$ denote the output distributions, where $\theta_f = \theta_b + M \odot (\theta_s - \theta_b)$. Assume the softmax mapping is differentiable. If updates are applied only to coordinates with importance score $\mathcal{I}_i \geq \tau$, then $D_{\mathrm{KL}}(q \,\|\, q_f) \leq D_{\mathrm{KL}}(q \,\|\, p)$.}

\begin{proof}
The reverse KL divergence $D_{\mathrm{KL}}(q \| r)$ is convex with respect to the second argument $r$. The fused distribution $q_f$ coincides with $q$ on all coordinates where the mask $M_i = 0$, and moves toward $p$ only where $M_i = 1$. 

By the construction of our importance score $\mathcal{I} = \theta_b \cdot \nabla \ell_{RKL}$, the selected coordinates are precisely those where the directional derivative of $D_{\mathrm{KL}}(q \| \cdot)$ along the vector $(p - q)$ is largest. This implies that the update yields the maximal reduction in divergence relative to the base model for a given perturbation magnitude. 

Since no update is applied to low-importance coordinates, $q_f$ lies in the linear subspace spanned by $q$ and the high-impact components of $p$. Formally, consider the first-order Taylor expansion of $D_{\mathrm{KL}}(q \| q_f)$ around $q$. Since $D_{\mathrm{KL}}(q \| q) = 0$ and the gradient is minimized at $q$, the divergence is dominated by the quadratic term of the perturbation distance. Since the perturbation in SCF-RKL is a sparse projection of the full perturbation $(\theta_s - \theta_b)$, it follows that the resulting divergence is strictly bounded by the divergence of the full secondary model $p$.
\end{proof}

\paragraph{Restatement of Theorem~\ref{thm:entropy_preservation} (Entropy Preservation).}
\textit{Let $H(\cdot)$ denote the Shannon entropy. Then there exists a constant $L > 0$ such that $H(q_f) \geq H(q) - L \cdot \| M \odot (\theta_s - \theta_b) \|_2$.}

\begin{proof}
The Shannon entropy $H(r) = -\sum_i r_i \log r_i$ is smooth and Lipschitz continuous over the interior of the probability simplex (bounded away from zero). Let $\delta = M \odot (\theta_s - \theta_b)$ represent the parameter perturbation vector.

First, by the Lipschitz continuity of the softmax function, there exists a constant $C$ such that:
\[
\|q_f - q\|_2 \leq C \|\theta_f - \theta_b\|_2 = C \|\delta\|_2.
\]
Next, since the entropy function $H$ is locally Lipschitz with constant $L'$ (assuming the distribution does not collapse to a deterministic vector), we have:
\[
|H(q_f) - H(q)| \leq L' \|q_f - q\|_2.
\]
Combining these inequalities:
\[
H(q) - H(q_f) \leq |H(q_f) - H(q)| \leq L' C \|\delta\|_2.
\]
Letting $L = L' C$, we obtain the lower bound:
\[
H(q_f) \geq H(q) - L \|\delta\|_2.
\]
This completes the proof.
\end{proof}

\subsection{Spectral Analysis and Perturbation Bounds}
\label{app:spectral_proofs}

In this section, we provide the rigorous basis for the geometric consistency observed in Section~\ref{subsec:geometric_analysis}. We model the fusion operation as a matrix perturbation $W_f = W_b + E$, where $E$ is the update matrix.

\paragraph{Restatement of Theorem~\ref{thm:subspace_rotation} (Bound on Principal Subspace Rotation).}
\textit{Let $W_b = U \Sigma V^T$ be the SVD of the base model. Let $\tilde{W}_f$ be the fused matrix with singular subspaces $\tilde{U}$. If the perturbation satisfies $\|E\|_2 < \delta$ (where $\delta$ is the minimum singular value gap), then the sine of the maximum principal angle $\Theta$ between $U$ and $\tilde{U}$ satisfies:}
\[
\sin \Theta_{\max}(U, \tilde{U}) \leq \frac{\| M \odot (W_s - W_b) \|_2}{\delta}.
\]

\begin{proof}
We invoke the \textbf{Wedin's $\sin \Theta$ Theorem}, which characterizes the sensitivity of singular subspaces to perturbations. For a generalized perturbation $E$, the theorem states:
\[
\max \{ \|\sin \Theta(U, \tilde{U})\|, \|\sin \Theta(V, \tilde{V})\| \} \leq C \frac{\|E\|_2}{\delta},
\]
where $C$ is a constant related to the projection alignment.

The crucial distinction between SCF-RKL and dense methods lies in the norm $\|E\|_2$.
\begin{enumerate}
    \item \textbf{Dense Fusion (e.g., Task Arithmetic):} The perturbation is $E_{\text{dense}} = \lambda(W_s - W_b)$. The norm $\|E_{\text{dense}}\|_2$ includes contributions from \textit{all} parameter differences. In Fine-tuning, many parameter updates correspond to random noise or overfitting in the null space of the data manifold. These accumulate to form a large perturbation vector, resulting in a large numerator and significant rotation (high $\sin \Theta$).
    
    \item \textbf{SCF-RKL:} The perturbation is $E_{\text{sparse}} = M \odot (W_s - W_b)$. The mask $M$ is derived from the Reverse-KL importance $\mathcal{I}$. By filtering based on $\mathcal{I}$, we implicitly zero-out updates in directions that have high magnitude but low information value (spectral noise). Consequently, $\|E_{\text{sparse}}\|_2 \ll \|E_{\text{dense}}\|_2$.
\end{enumerate}

Since the bound depends linearly on $\|E\|_2$, the sparse update of SCF-RKL theoretically guarantees that the principal angles remain small, preserving the geometric structure of the base model.
\end{proof}

\begin{corollary}[Minimization of Normalized Spectral Shift]
The Normalized Spectral Shift (NSS) is defined as $\frac{\|\sigma(W_f) - \sigma(W_b)\|_2}{\|\sigma(W_b)\|_2}$. By Weyl's Inequality, $|\sigma_i(W_f) - \sigma_i(W_b)| \leq \|E\|_2$. 
Since SCF-RKL retains only the tail of the importance distribution (top-$k$ parameters), it minimizes the Frobenius norm $\|E\|_F$ compared to any dense combination. Thus, it directly minimizes the upper bound of NSS:
\[
\text{NSS}(W_f) \leq \frac{\| M \odot (W_s - W_b) \|_F}{\|W_b\|_F} \ll \frac{\| W_s - W_b \|_F}{\|W_b\|_F}.
\]
This provides the rigorous justification for the order-of-magnitude reduction in NSS observed in Figure \ref{fig:nss_all}.
\end{corollary}

\subsection{Discrete Composition and Probabilistic Bounds}
\label{app:discrete_analysis}

\subsubsection{Statistical Bound on Base Retention}
Here we formalize the "Base Model Dominance" discussed in Section~\ref{subsec:discrete_composition}. We model the parameter importance scores $\mathcal{I}$ as following a heavy-tailed distribution, reflecting the sparsity of "knowledge" in LLMs.

\begin{proposition}[Base Model Dominance]
\label{prop:base_ratio}
Assume the importance scores follow an Exponential distribution $\mathcal{I} \sim \mathrm{Exp}(\lambda)$. Let the mask threshold $\tau$ be determined by the standard IQR rule: $\tau = Q_2 + \alpha(Q_3 - Q_1)$. The proportion of parameters retained from the base model, $\rho_{base}$, is bounded by:
\[
\rho_{base} = P(\mathcal{I} < \tau) \geq 1 - e^{-\lambda (Q_2 + \alpha(Q_3 - Q_1))}.
\]
\end{proposition}

\begin{proof}
The cumulative distribution function (CDF) of the exponential distribution is $F(x) = 1 - e^{-\lambda x}$. The proportion of parameters below the threshold $\tau$ is simply $F(\tau)$. 
Since the IQR rule with $\alpha=1.5$ typically sets a threshold far into the tail of the distribution, the term $e^{-\lambda \tau}$ becomes negligible. For instance, if $\tau$ corresponds to the 95th percentile, then $\rho_{base} = 0.95$. This guarantees that the fused model $M_f$ is a minimal perturbation of $M_0$.
\end{proof}

\subsubsection{Perturbation Norm Comparison: Dense vs. Sparse}
The geometric stability of SCF-RKL is rooted in the difference between the perturbation norms of dense and sparse updates.

\paragraph{Dense Linear Merging.}
Consider a standard merging update (e.g., Task Arithmetic) with scaling factor $\lambda$: $\theta_{new} = \theta_b + \lambda(\theta_s - \theta_b)$. The squared Euclidean norm of the update vector is:
\[
\|E_{linear}\|_2^2 = \sum_{j=1}^d \lambda^2 (\theta_{s,j} - \theta_{b,j})^2 = \lambda^2 \sum_{j=1}^d \Delta_j^2.
\]
Assuming the differences $\Delta_j$ have variance $\sigma^2$, the expected norm scales as $\lambda^2 d \sigma^2$. Even for small $\lambda$, the factor $d$ (billions of parameters) dominates.

\paragraph{Sparse Fusion (SCF-RKL).}
Our update is $M \odot (\theta_s - \theta_b)$. The norm sums only over the non-zero indices of $M$, denoted as set $K$, where $|K| = k$:
\[
\|E_{sparse}\|_2^2 = \sum_{j \in K} (\theta_{s,j} - \theta_{b,j})^2.
\]
Since the mask selects only the most important parameters, $k \ll d$. Typically, $k \approx 0.01 d$ to $0.05 d$. Even though the individual differences $\Delta_j$ in the active set might be larger than average (since they are important), the massive reduction in the count of perturbed terms ($k$ vs $d$) ensures that $\|E_{sparse}\|_2 \ll \|E_{linear}\|_2$. 
This reduced norm directly limits the subspace rotation $\sin \Theta$ (via Wedin's Theorem, Appendix B), providing the theoretical justification for the stability observed in Figure \ref{fig:rotation_analysis}.

\subsection{Accuracy Change by Blocking Layers}

\begin{figure*}[tbp]
    \centering
    \includegraphics[width=0.48\linewidth]{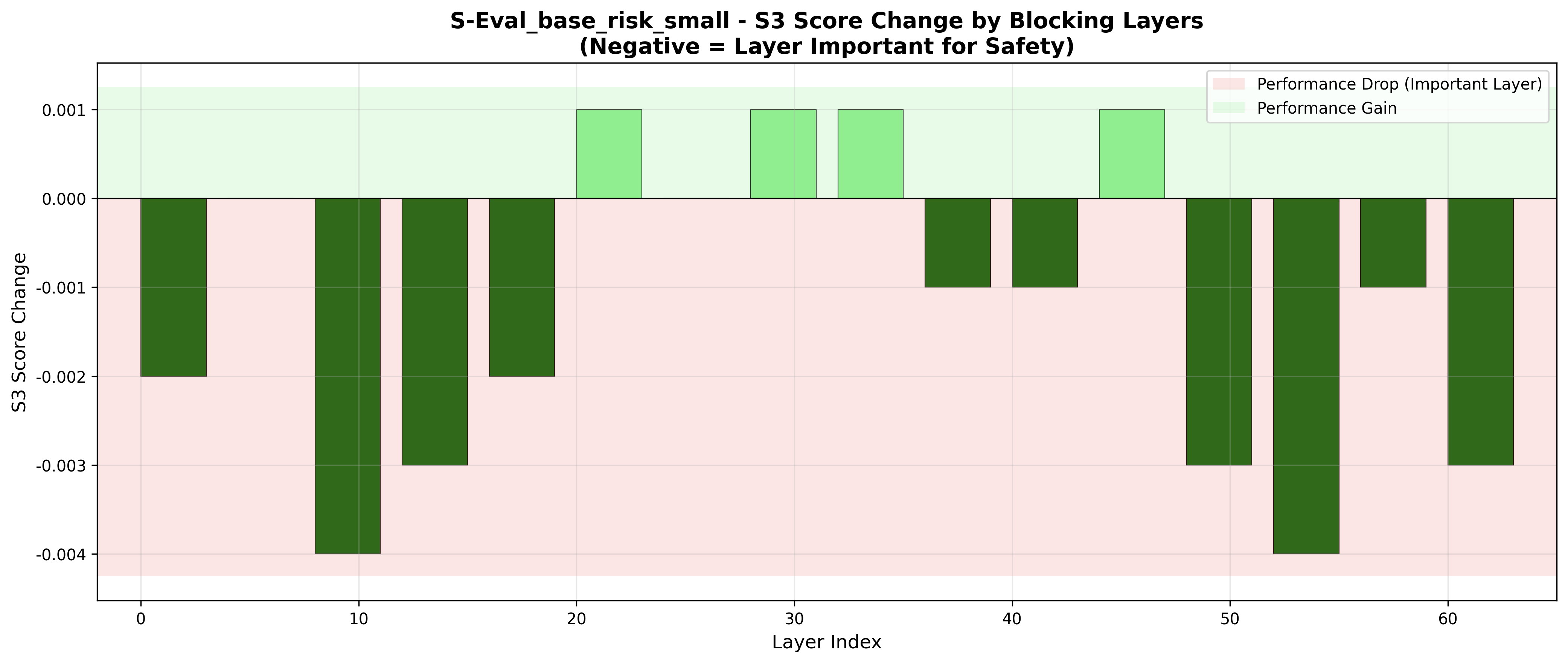}
    \hfill
    \includegraphics[width=0.48\linewidth]{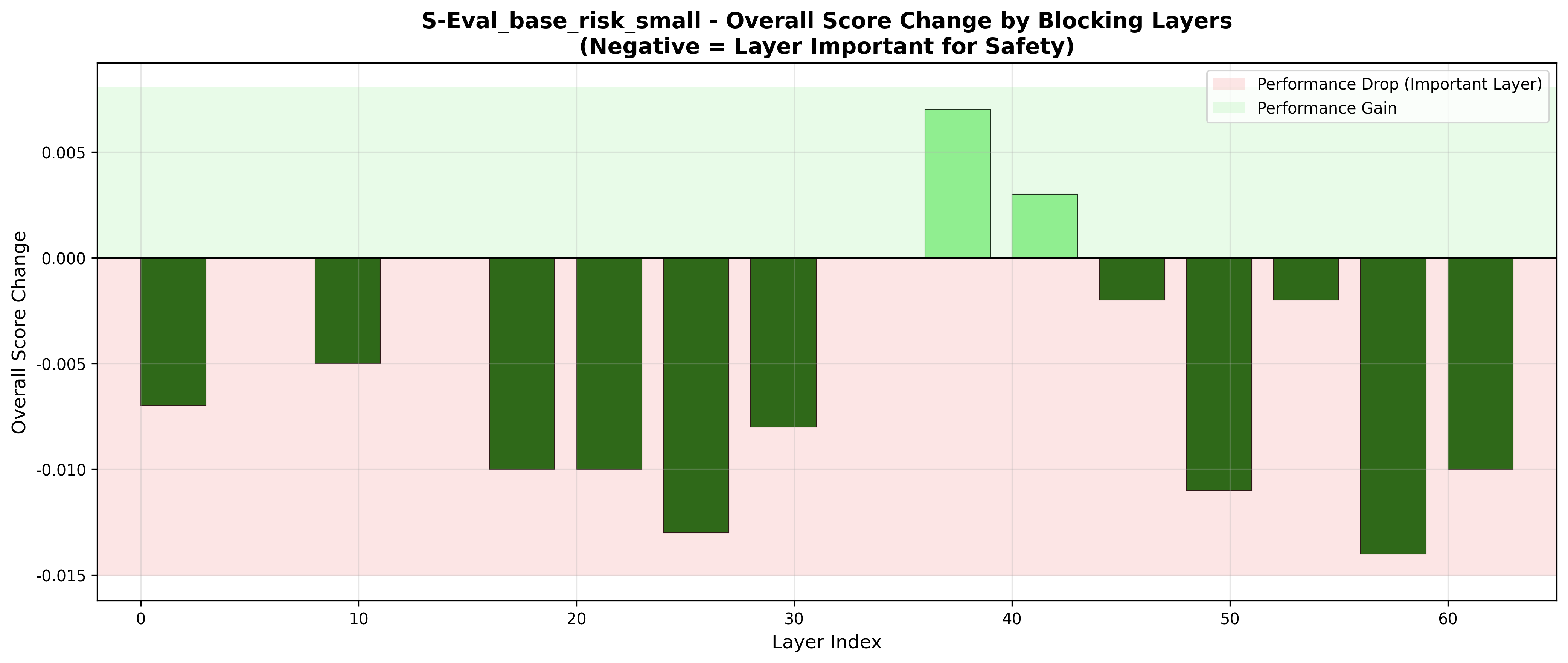}
    \hfill
    \includegraphics[width=0.48\linewidth]{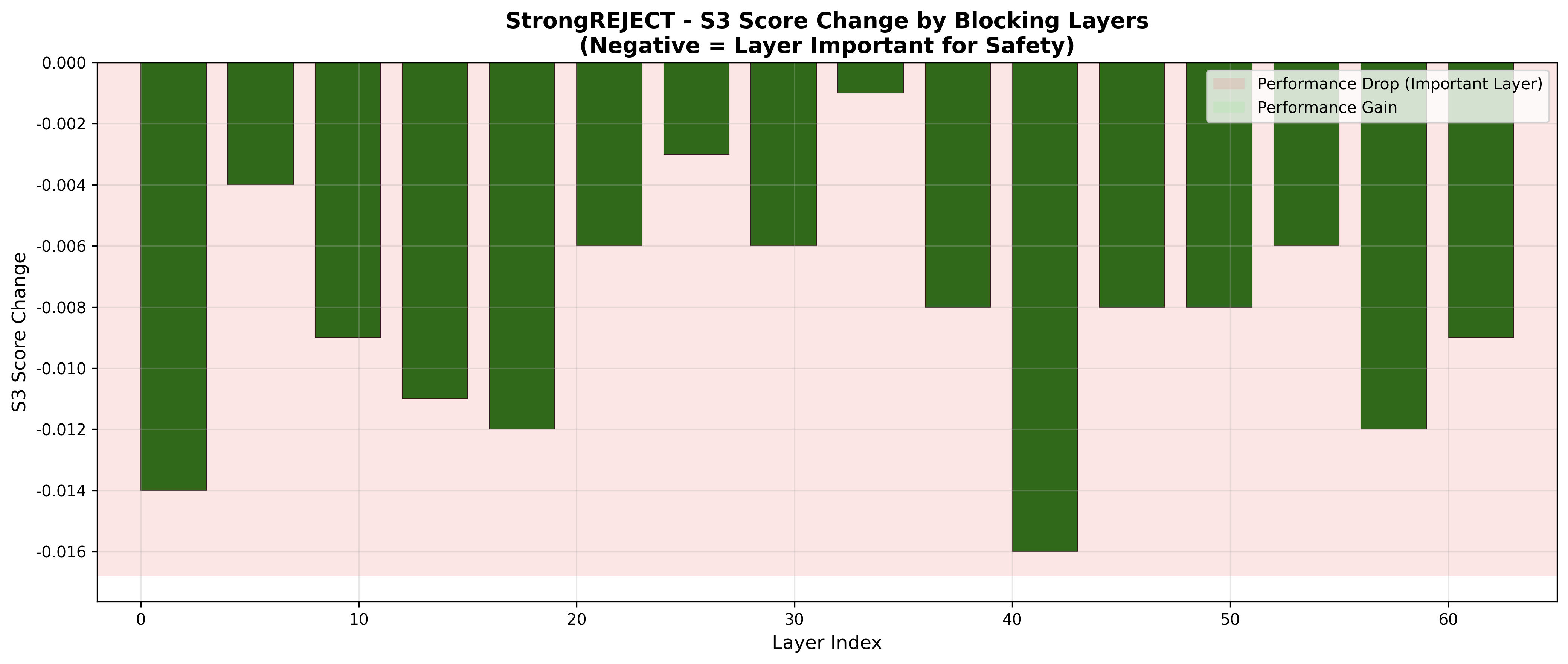} 
    \hfill
    \includegraphics[width=0.48\linewidth]{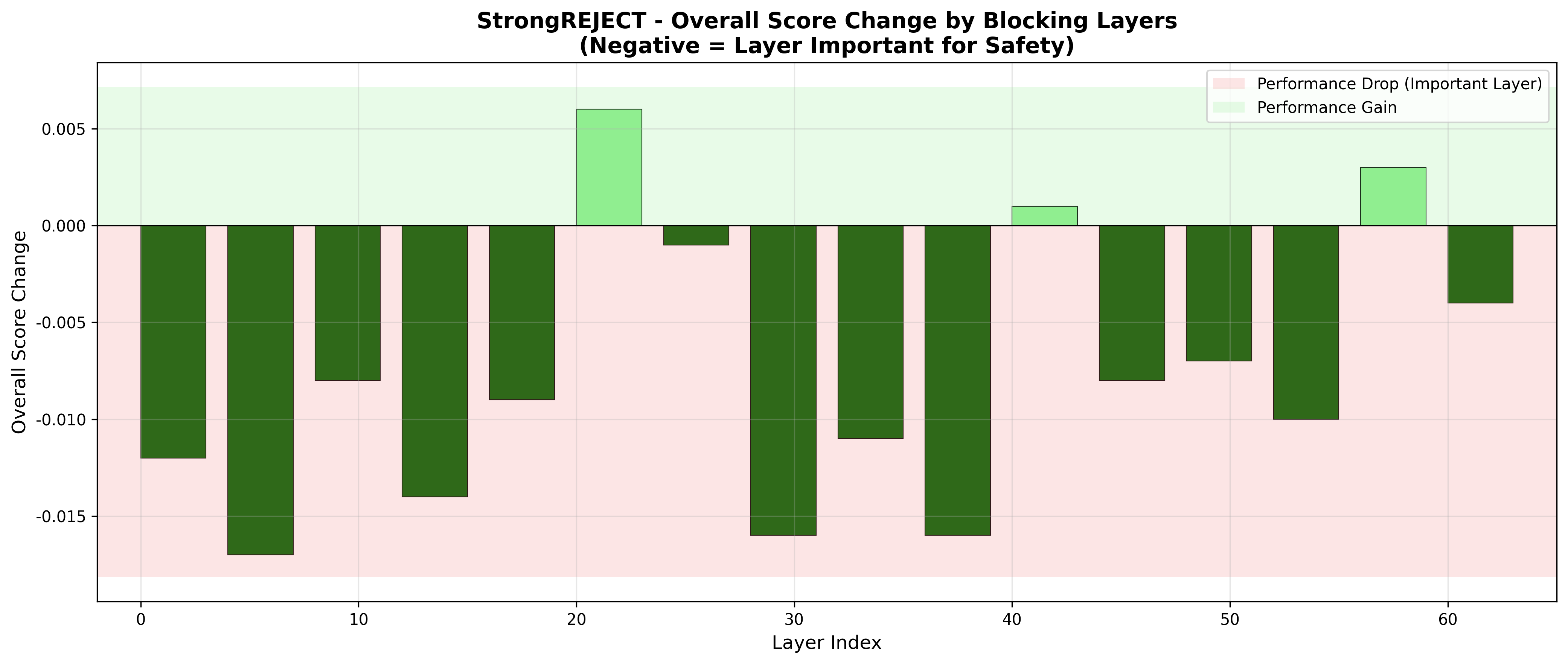}
    \hfill
    \includegraphics[width=0.48\linewidth]{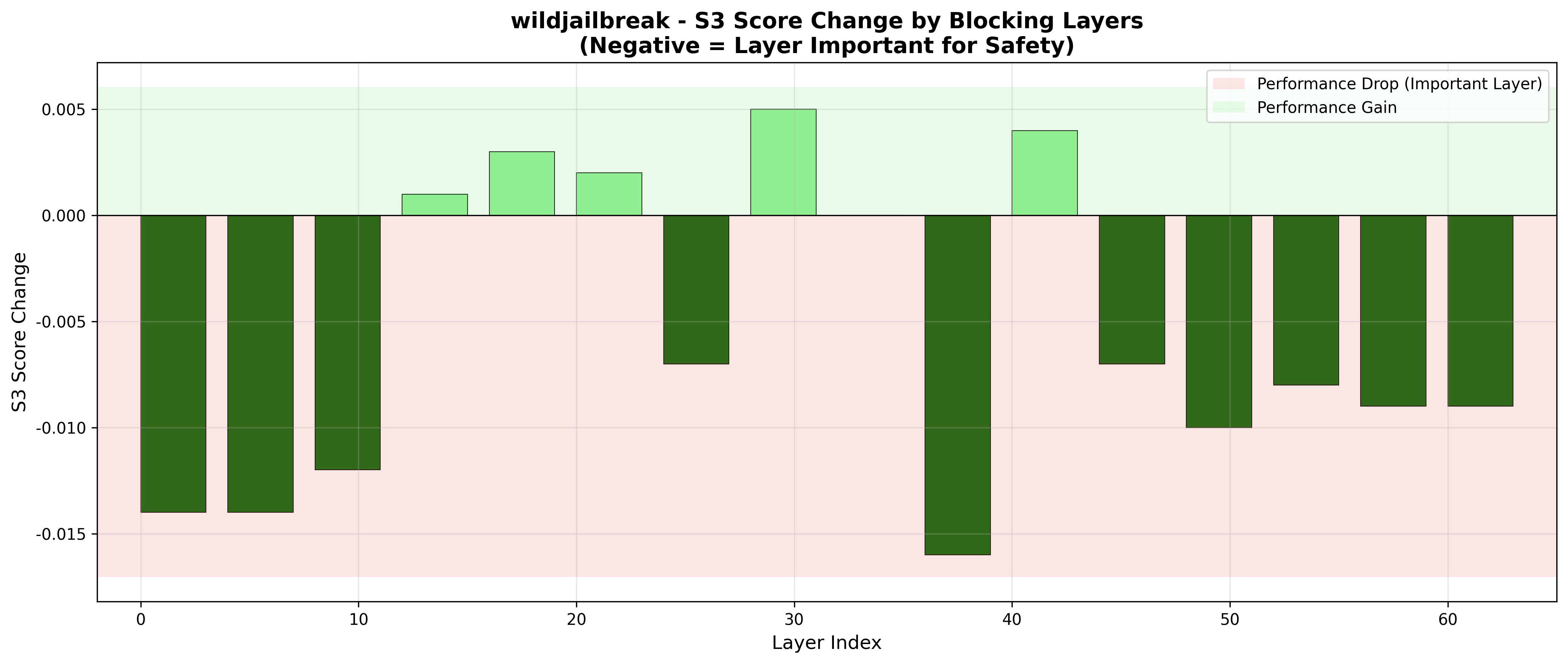} 
    \hfill
    \includegraphics[width=0.48\linewidth]{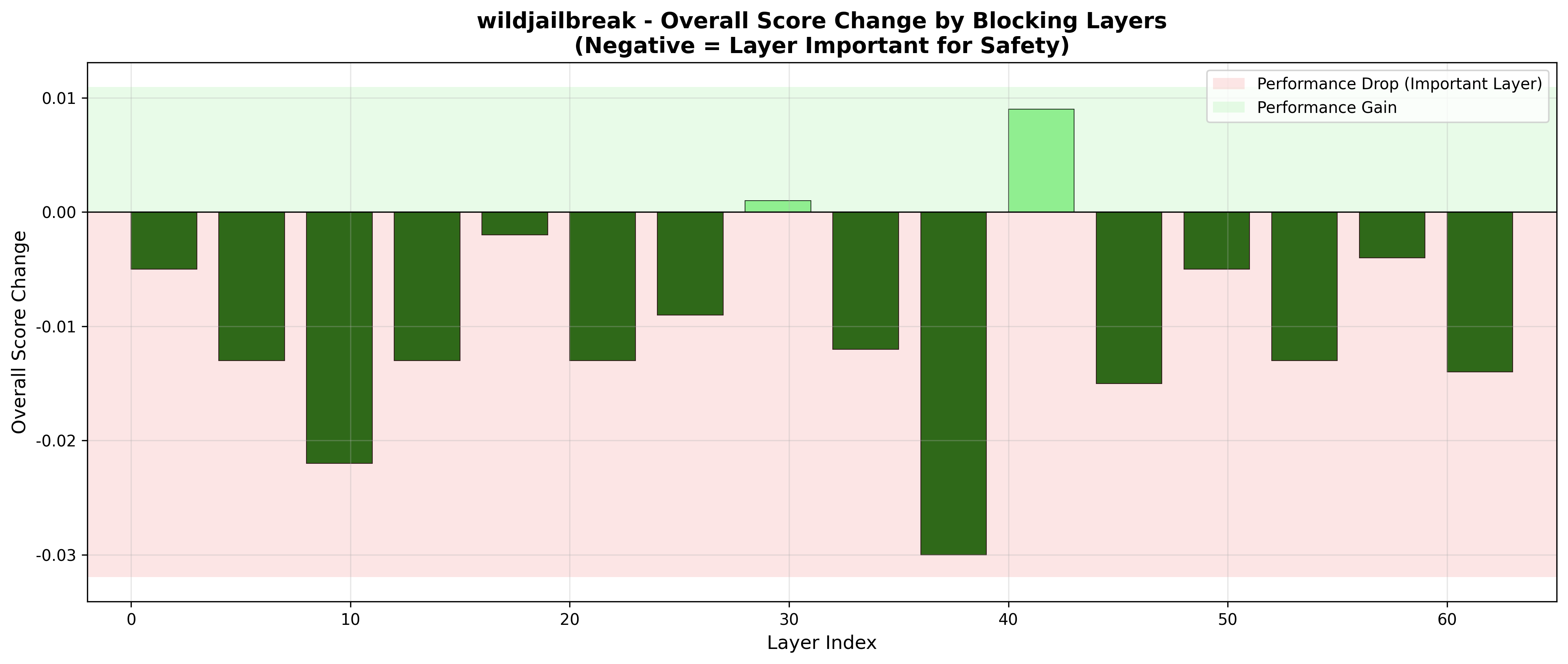}
    \caption{Accuracy Change by Blocking Layers on Safety Benchmark.}
    \label{fig:accuracy_change_safety}
\end{figure*}

Figure~\ref{fig:accuracy_change_safety} presents a layer-wise \emph{blocking} diagnostic on three safety benchmarks (S-Eval, StrongREJECT, and WildJailbreak), each evaluated with both the overall metric and the S3 subset. Concretely, starting from the merged model $M_f$, we \emph{restore} the parameters of selected transformer layers to the base model $M_0$ (i.e., replace the fused layer with the corresponding layer from $M_0$) while keeping all other layers unchanged, and then measure the score change relative to the original merged model. A negative score change indicates that blocking (restoring to $M_0$) hurts performance, implying that the fused parameters in that layer contribute positively to the merged model's safety behavior.

Across all three benchmarks, most blocked layers lead to a consistent score drop, and this trend holds for both S3 and overall metrics. This suggests that the safety improvements are not driven by a single isolated component, but are instead distributed across multiple layers in the fused model. Meanwhile, the magnitude of degradation is non-uniform across depth: some layers cause only mild changes when restored, whereas others induce a noticeably larger drop, indicating that the merge injects more functionally meaningful safety-relevant updates in those layers. We also observe a few layers with slight positive changes, which likely correspond to noisy or task-irrelevant perturbations where restoring $M_0$ happens to be beneficial. Overall, the predominance of negative changes in Figure~\ref{fig:accuracy_change_safety} validates that our merging procedure produces real, layer-localized gains on safety benchmarks, rather than superficial improvements attributable to evaluation variance.

\section{Additional Experimental Figures}

\subsection{Per-benchmark accuracy change of fused models at diverse scales}

\begin{figure}[ht]
  \vskip 0.2in
  \begin{center}
    \centerline{\includegraphics[width=0.5\columnwidth]{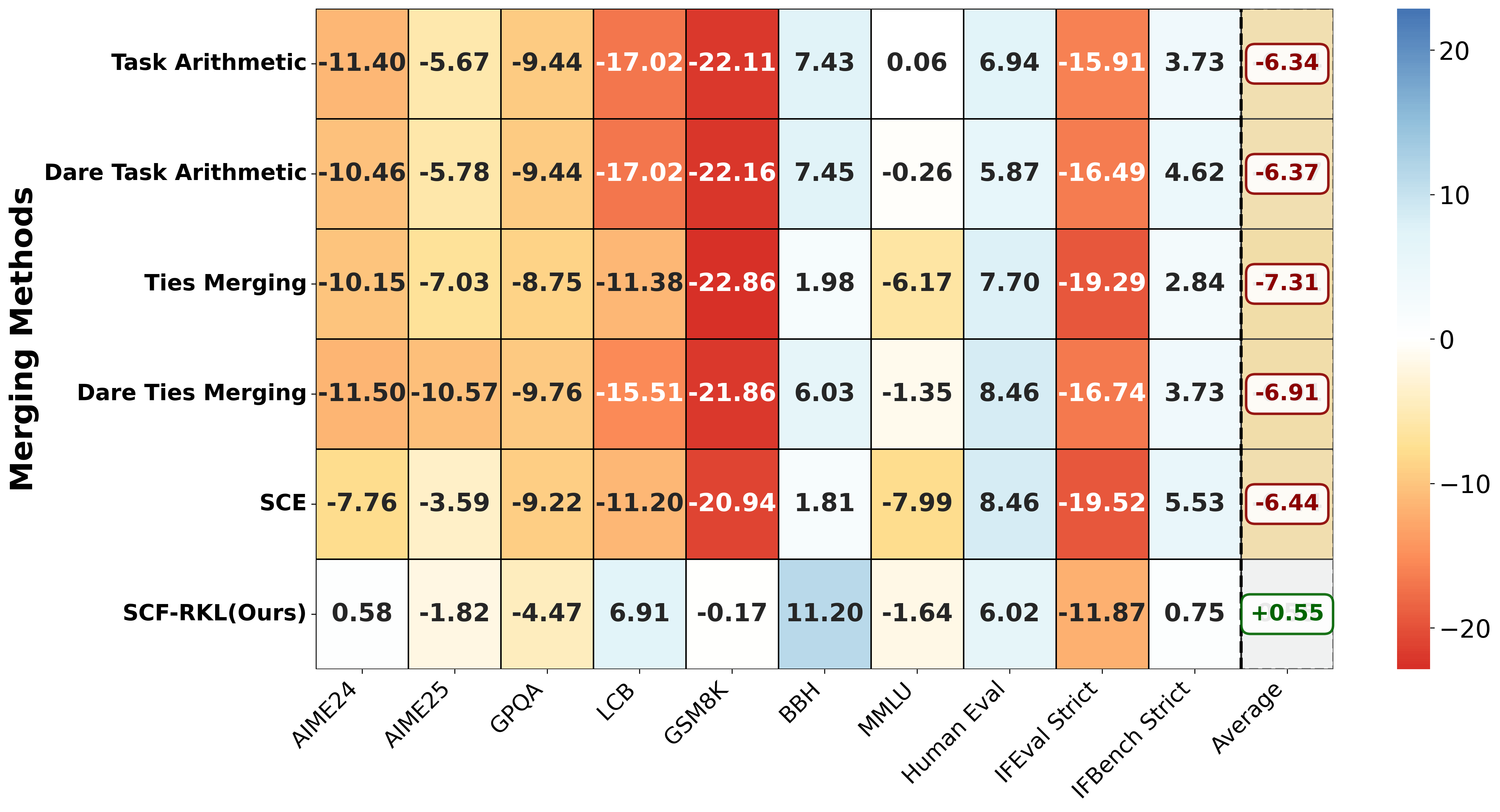}}
    \caption{
      Per-benchmark accuracy change of fused models at the 14B scale, measured as the difference between the fused model’s score and the average score of the two base models.Positive values (blue) indicate improvement; negative values (red) indicate degradation.
    }
    \label{fig:heatmap accuracy_changes_after_merging-14b}
  \end{center}
\end{figure}

\subsection{Repetition Rate of fused models at diverse scales}

\begin{figure}[ht]
  \vskip 0.2in
  \begin{center}
    \centerline{\includegraphics[width=0.5\columnwidth]{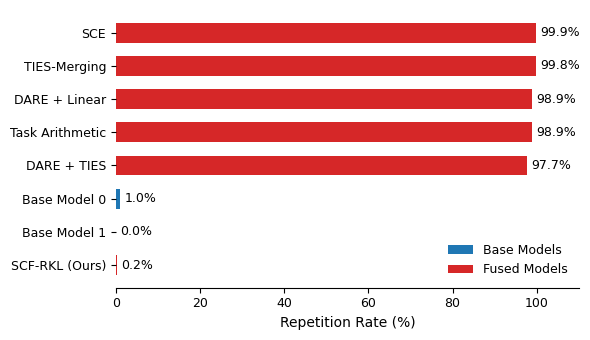}}
    \caption{
      Repetition Rate of 32B models on GSM8K.
    }
    \label{fig: Repetition Rate of 32B models on GSM8K}
  \end{center}
\end{figure}

\begin{figure}[ht]
  \vskip 0.2in
  \begin{center}
    \centerline{\includegraphics[width=0.5\columnwidth]{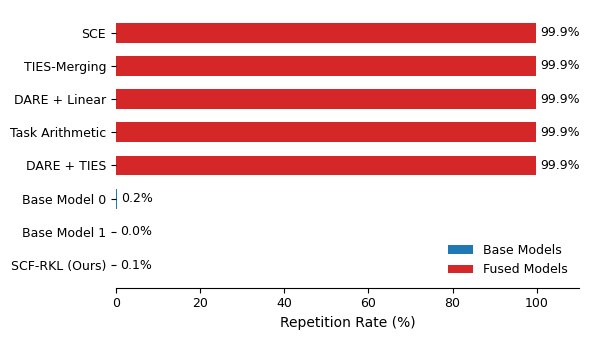}}
    \caption{
      Repetition Rate of 14B models on GSM8K.
    }
    \label{fig: Repetition Rate of 14B models on GSM8K}
  \end{center}
\end{figure}

\begin{figure}[ht]
  \vskip 0.2in
  \begin{center}
    \centerline{\includegraphics[width=0.5\columnwidth]{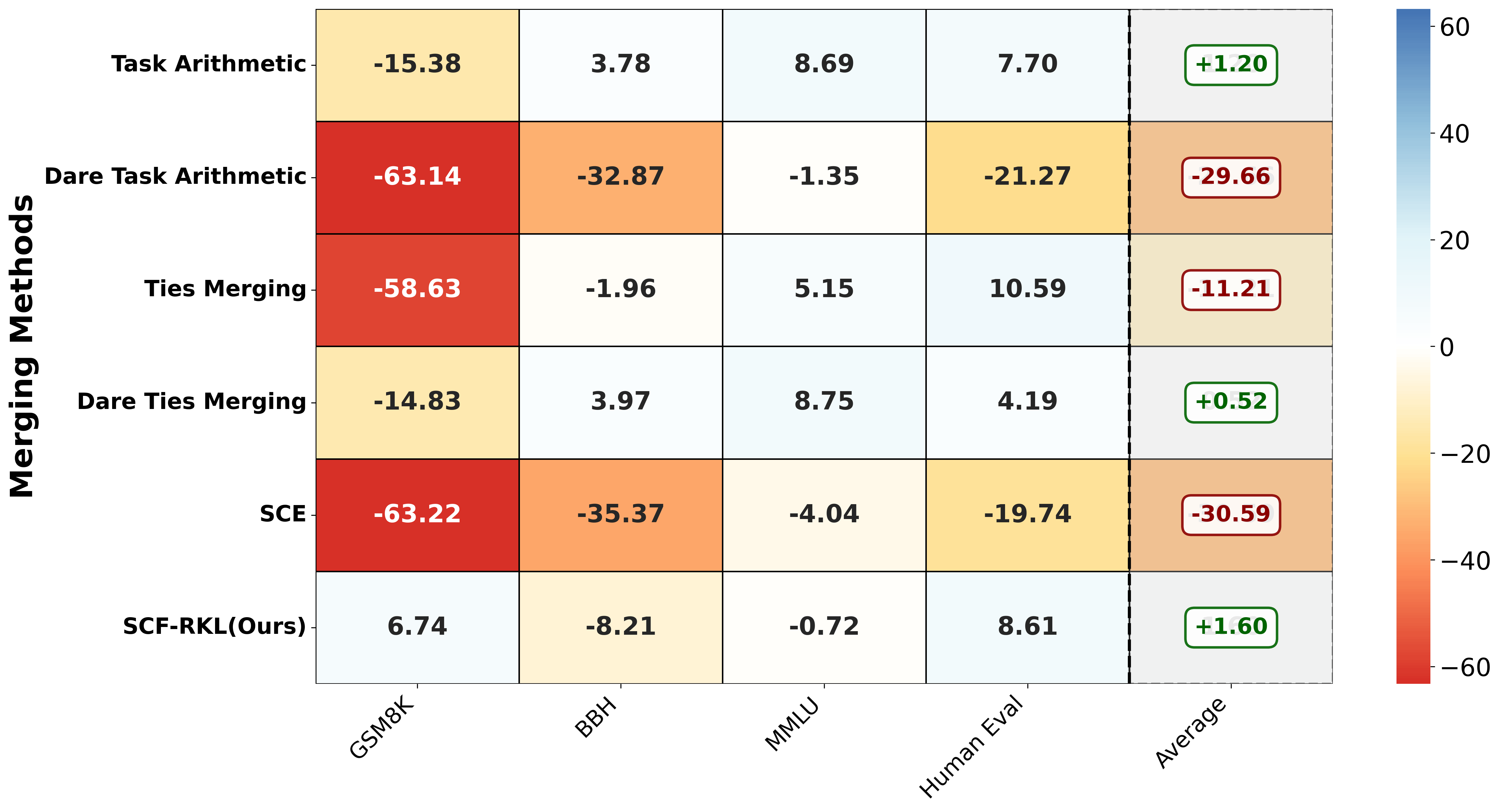}}
    \caption{
      Per-benchmark accuracy change of fused models at the
8B scale, measured as the difference between the fused model’s
score and the average score of the two base models.
    }
    \label{fig: heatmap accuracy_changes_after_merging-8b }
  \end{center}
\end{figure}

\begin{figure}[ht]
  \vskip 0.2in
  \begin{center}
    \centerline{\includegraphics[width=0.5\columnwidth]{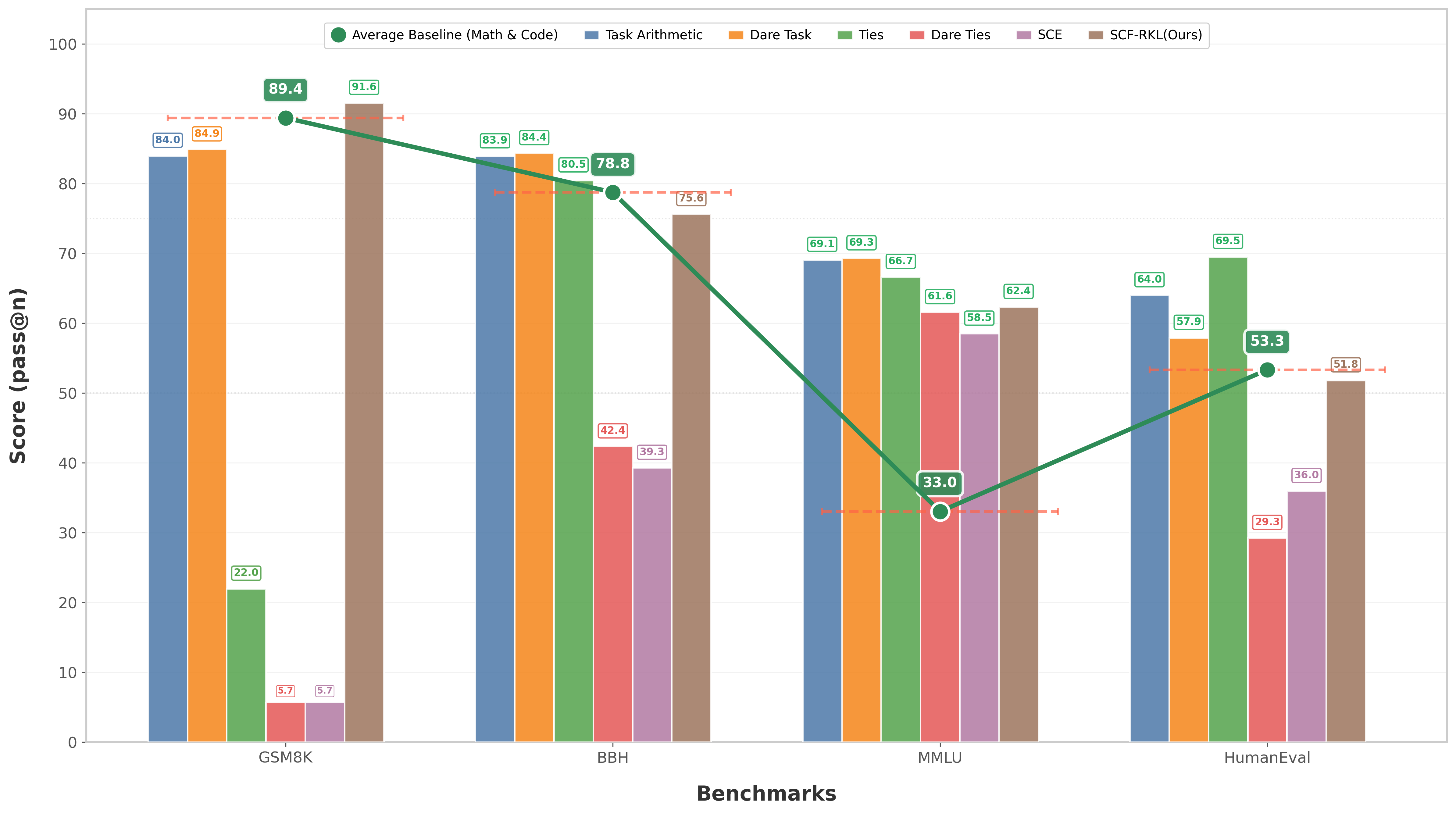}}
    \caption{Comparison of Pass@n Scores Between Base Models (Lines) and Fused Models (Bars) on LLaMA-3-8B-Instruct–Derived Models.}
    \label{fig: bar line at passn 8b}
  \end{center}
\end{figure}

\begin{figure}[ht]
  \vskip 0.2in
  \begin{center}
    \centerline{\includegraphics[width=0.5\columnwidth]{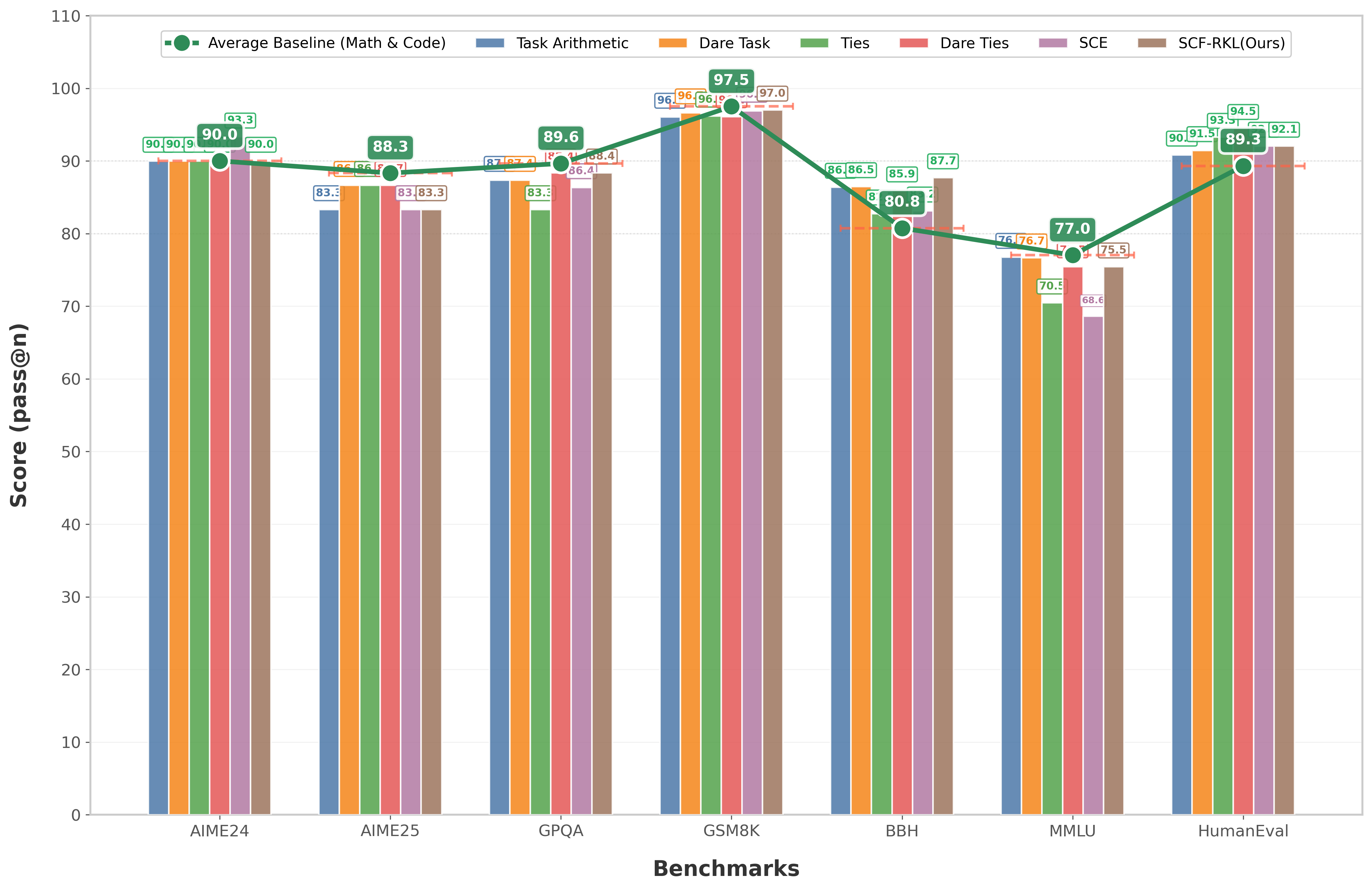}}
    \caption{Comparison of Pass@n Scores Between Base Models (Lines) and Fused Models (Bars) on Deepseek-R1-Distill-Qwen2.5-14B–Derived Models.
    }
    \label{fig: bar line passn 14b}
  \end{center}
\end{figure}

\subsection{Comparison of fuse models' upper bound}

\begin{figure}[ht]
  \vskip 0.2in
  \begin{center}
    \centerline{\includegraphics[width=0.5\columnwidth]{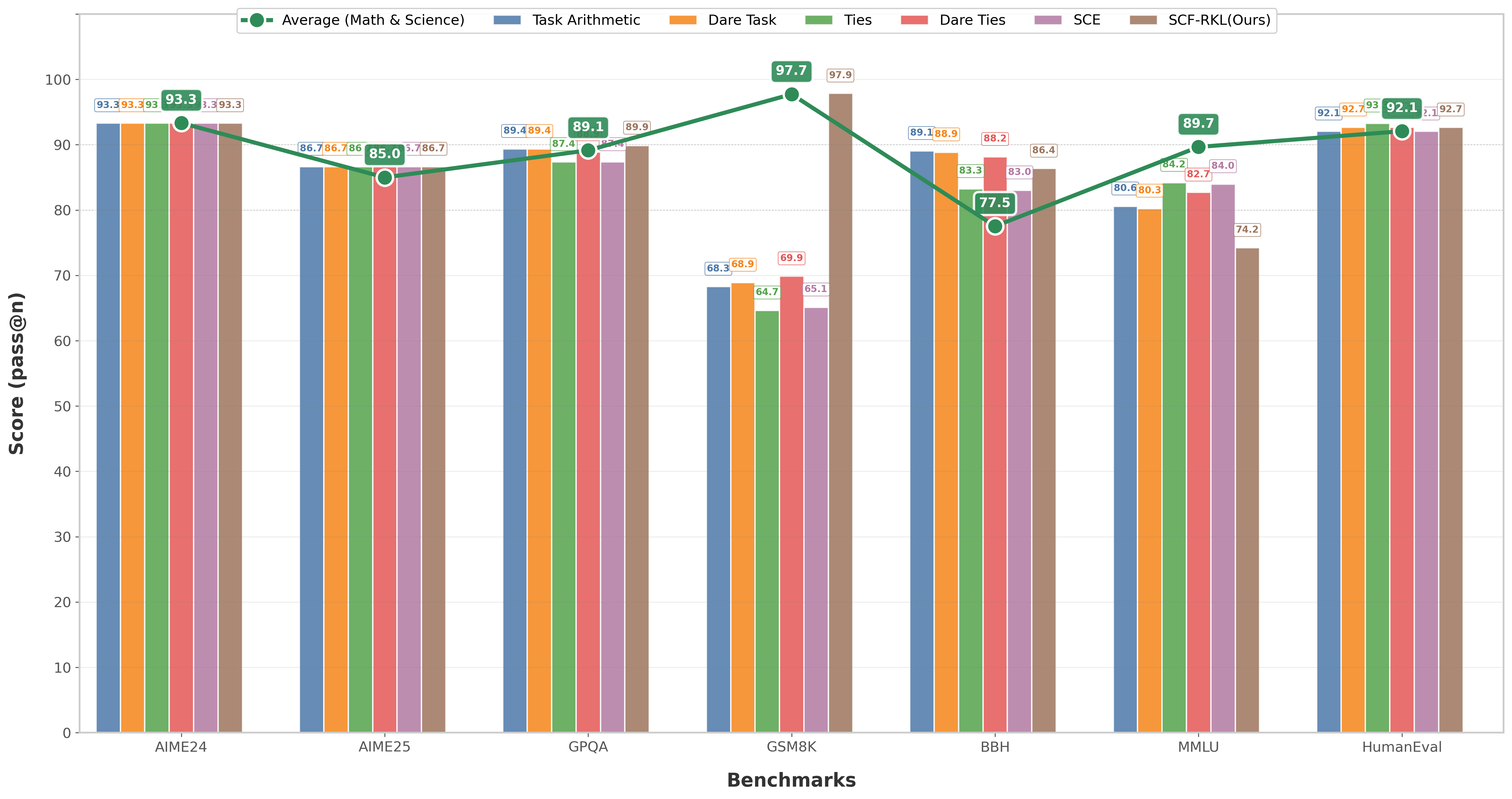}}
    \caption{Comparison of Pass@n Scores Between Base Models (Lines) and Fused Models (Bars) on Deepseek-R1-Distill-Qwen2.5-32B–Derived Models.
    }
    \label{fig: bar line passn 32b}
  \end{center}
\end{figure}

\begin{table}[h]
\centering
\caption{Performance of merging TinyR1-Preview-Math-32B (Math), and TinyR1-Preview-Science-32B (Science) 
on all datasets. 
*: The merged model fails to provide structured response.
All metrics are reported as \texttt{pass@n} scores unless otherwise specified.}
\label{tab:tiny-r1-preview-32b-merging-passn}
\begin{tabularx}{\textwidth}{@{}lccccccccc@{}}
\toprule
\multirow{2}{*}{Merging Methods} 
& \multirow{2}{*}{Models} 
& \multicolumn{3}{c}{Advanced Reasoning} 
& \multicolumn{3}{c}{\makecell{General Knowledge \& Reasoning}}
& \multicolumn{1}{c}{Code Gen} \\
\cmidrule(lr){3-5}
\cmidrule(lr){6-8}
\cmidrule(lr){9-9}
& 
& \makecell{AIME24$\uparrow$} 
& \makecell{AIME25$\uparrow$} 
& \makecell{GPQA$\uparrow$} 
& \makecell{GSM8K$\uparrow$} 
& \makecell{BBH$\uparrow$} 
& \makecell{MMLU$\uparrow$} 
& \makecell{HumanEval$\uparrow$} \\
\midrule
\multirow{2}{*}{w/o Merging} 
& Math
& 93.33
& 86.67
& 88.89
& 97.58
& 72.97
& 84.99
& 92.07 \\
& Science
& 93.33
& 83.33
& 89.39
& 97.88
& 82.09
& 94.33
& 92.07 \\
\midrule
Task Arithmetic & Fuse & 93.33 & 86.67 & 89.39 & 68.31 & 89.08 & 80.59 & 92.07 \\
Dare Task Arithmetic & Fuse & 93.33 & 86.67 & 89.39 & 68.92 & 88.87 & 80.26 & 92.68 \\
Ties Merging & Fuse & 93.33 & 86.67 & 87.37 & 64.67 & 83.26 & 84.21 & 93.29 \\
Dare Ties Merging & Fuse & 93.33 & 86.67 & 88.89 & 69.90 & 88.17 & 82.72 & 92.68 \\
SCE & Fuse & 93.33 & 86.67 & 87.37 & 65.13 & 83.05 & 84.00 & 92.07 \\
SCF-RKL (Ours)& Fuse & 93.33 & 86.67 & 89.90 & 97.88 & 86.38 & 74.24 & 92.68 \\
\bottomrule
\end{tabularx}
\end{table}

\begin{table}[h]
\centering
\caption{Performance of merging TinyR1-Preview-Math-14B (Math), and TinyR1-Preview-Code-14B (Code) 
on all datasets.  
*: The merged model fails to provide structured response.
All metrics are reported as \texttt{pass@n} scores unless otherwise specified.}
\label{tab:tiny-r1-preview-14b-merging-passn}
\begin{tabularx}{\textwidth}{@{}lccccccccc@{}}
\toprule
\multirow{2}{*}{Merging Methods} 
& \multirow{2}{*}{Models} 
& \multicolumn{3}{c}{Advanced Reasoning} 
& \multicolumn{3}{c}{General Knowledge \& Reasoning} 
& \multicolumn{1}{c}{Code Gen} \\
\cmidrule(lr){3-5}
\cmidrule(lr){6-8}
\cmidrule(lr){9-9}
& 
& \makecell{AIME24$\uparrow$} 
& \makecell{AIME25$\uparrow$} 
& \makecell{GPQA$\uparrow$} 
& \makecell{GSM8K$\uparrow$} 
& \makecell{BBH$\uparrow$} 
& \makecell{MMLU$\uparrow$} 
& \makecell{HumanEval$\uparrow$} \\
\midrule
\multirow{2}{*}{w/o Merging} 
& Math
& 90.00
& 90.00
& 88.89
& 97.50
& 74.37
& 74.11
& 87.80 \\
& Code
& 90.00
& 86.67
& 90.40
& 97.50
& 87.14
& 79.97
& 90.85 \\
\midrule
Task Arithmetic & Fuse & 90.00 & 83.33 & 87.37 & 96.06 & 86.42 & 76.77 & 90.85 \\
Dare Task Arithmetic & Fuse & 90.00 & 86.67 & 87.37 & 96.66 & 86.50 & 76.69 & 91.46 \\
Ties Merging & Fuse & 90.00 & 86.67 & 83.33 & 96.21 & 82.74 & 70.49 & 93.29 \\
Dare Ties Merging & Fuse & 90.00 & 86.67 & 88.38 & 96.13 & 85.92 & 75.47 & 94.51 \\
SCE & Fuse & 93.33 & 83.33 & 86.36 & 96.89 & 83.15 & 68.64 & 92.07 \\
SCF-RKL (Ours)& Fuse & 90.00 & 83.33 & 88.38 & 97.04 & 87.72 & 75.45 & 92.07 \\
\bottomrule
\end{tabularx}
\end{table}

\begin{table}[h]
\centering
\caption{Performance of merging Meta-Llama-3-8B-Instruct (Base 0), and MAmmoTH2-8B-Plus(Base 1)
on all datasets. 
*: The merged model fails to provide structured response.
All metrics are reported as \texttt{pass@n} scores unless otherwise specified.}
\label{tab:tiny-r1-preview-8b-merging}
\begin{tabular}{@{}lcccccc@{}}
\toprule
\multirow{2}{*}{Merging Methods} 
& \multirow{2}{*}{Models}  
& \multicolumn{3}{c}{General Knowledge \& Reasoning} 
& \multicolumn{1}{c}{Code Generation} \\
\cmidrule(lr){3-5}
\cmidrule(lr){6-6}
& 
& GSM8K$\uparrow$ 
& BBH$\uparrow$ 
& MMLU$\uparrow$ 
& HumanEval$\uparrow$ \\
\midrule
\multirow{2}{*}{w/o Merging} 
& Base 0
& 92.72
& 81.50
& 66.10
& 62.80 \\
& Base 1
& 86.13
& 76.03
& 0.00
& 43.90 \\
\midrule
Task Arithmetic & Fuse & 84.00 & 83.92 & 69.10 & 64.02 \\
Dare Task Arithmetic & Fuse & 84.91 & 84.38 & 69.31 & 57.93 \\
Ties Merging & Fuse & 21.99 & 80.47 & 66.65 & 69.51 \\
Dare Ties Merging & Fuse & 5.69 & 42.37 & 61.60 & 29.27 \\
SCE & Fuse & 5.69 & 39.30 & 58.52 & 35.98 \\
SCF-RKL (Ours) & Fuse & 91.58 & 75.64 & 62.35 & 51.83 \\
\bottomrule
\end{tabular}
\end{table}

\begin{table}[t]
\centering
\caption{Performance of merging Mistral-7B-Instruct-v0.2(Base 0),MathCoder2-Mistral-7B(Base 1)
on all datasets. 
*: The merged model fails to provide structured response.
All metrics are reported as \texttt{pass@n} scores unless otherwise specified.}
\label{tab:tiny-r1-preview-8b-merging}
\begin{tabular}{@{}lcccccc@{}}
\toprule
\multirow{2}{*}{Merging Methods} 
& \multirow{2}{*}{Models}  
& \multicolumn{3}{c}{General Knowledge \& Reasoning} 
& \multicolumn{1}{c}{Code Generation} \\
\cmidrule(lr){3-5}
\cmidrule(lr){6-6}
& 
& GSM8K$\uparrow$ 
& BBH$\uparrow$ 
& MMLU$\uparrow$ 
& HumanEval$\uparrow$ \\
\midrule
\multirow{2}{*}{w/o Merging} 
& Base 0
& 68.23
& 66.36
& 51.99
& 16.46 \\
& Base 1
& 45.19
& 73.03
& 33.53
& 0.61 \\
\midrule
Task Arithmetic & Fuse & 78.70 & 71.61 & 44.14 & 39.02 \\
Dare Task Arithmetic & Fuse & 79.76 & 71.38 & 43.03 & 39.63 \\
Ties Merging & Fuse & 44.35 & 72.85 & 33.85 & 0.00 \\
Dare Ties Merging & Fuse & 81.80 & 75.39 & 56.09 & 36.59 \\
SCE & Fuse & 43.67 & 72.41 & 33.76 & 0.00 \\
SCF-RKL (Ours) & Fuse & 82.79 & 74.58 & 56.96 & 41.46 \\
\bottomrule
\end{tabular}
\end{table}

\begin{figure}[ht]
  \vskip 0.2in
  \begin{center}
    \centerline{\includegraphics[width=\columnwidth]{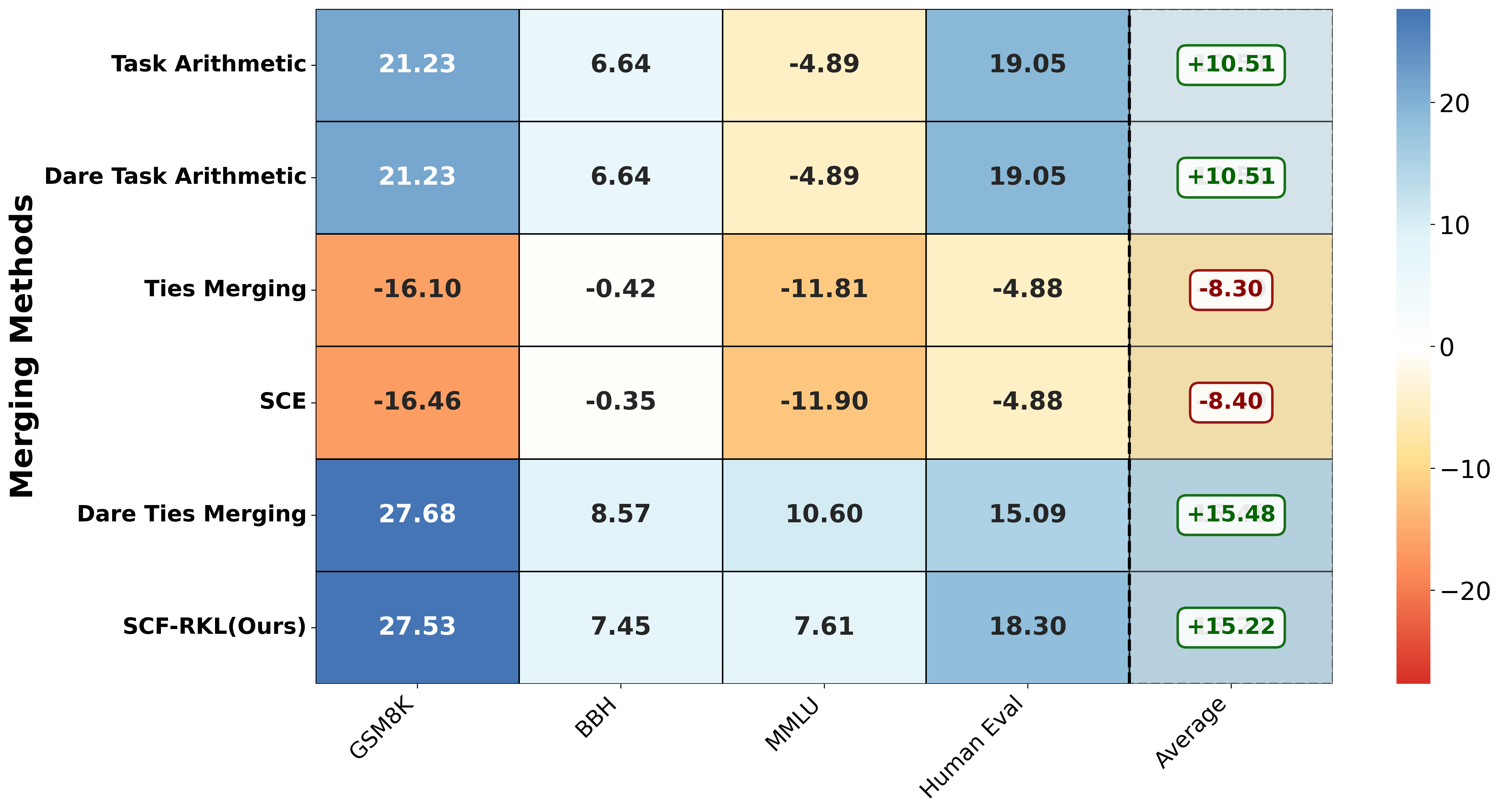}}
    \caption{
      Per-benchmark accuracy change of fused models at the
7B scale, measured as the difference between the fused model’s
score and the average score of the two base models.Positive values
(blue) indicate improvement; negative values (red) indicate degra-
dation.
    }
    \label{fig: heatmpa accuracy_changes_after_merging-7b}
  \end{center}
\end{figure}

\begin{table*}[t]
\centering
\caption{Performance of merging TinyR1-Preview-Math-32B (Math), and TinyR1-Preview-Science-32B (Sci) on all datasets. All metrics are reported as $\text{pass}@1$ scores. 
}
\label{tab:tiny-r1-preview-32b-merging-ablation}
\setlength{\tabcolsep}{2.5pt} 
\begin{tabularx}{\textwidth}{@{}l
>{\centering\arraybackslash}X
>{\centering\arraybackslash}X
>{\centering\arraybackslash}X
>{\centering\arraybackslash}X
>{\centering\arraybackslash}X
>{\centering\arraybackslash}X
>{\centering\arraybackslash}X
>{\centering\arraybackslash}X
>{\centering\arraybackslash}X
>{\centering\arraybackslash}X
>{\centering\arraybackslash}X
@{}}
\toprule
\multirow{2}{*}{Methods} 
& \multicolumn{4}{c}{Advanced Reasoning} 
& \multicolumn{3}{c}{Gen. Know. \& Reasoning} 
& \multicolumn{1}{c}{Code Gen} 
& \multicolumn{2}{c}{Inst. Following}
& \multirow{2}{*}{Avg} \\
\cmidrule(lr){2-5}
\cmidrule(lr){6-8}
\cmidrule(lr){9-9}
\cmidrule(lr){10-11}
& \makecell{AIME\\24$\uparrow$} 
& \makecell{AIME\\25$\uparrow$} 
& \makecell{GPQA\\$\uparrow$} 
& \makecell{LCB\\$\uparrow$} 
& \makecell{GSM\\8K$\uparrow$} 
& \makecell{BBH\\$\uparrow$} 
& \makecell{MMLU\\$\uparrow$} 
& \makecell{Human\\Eval$\uparrow$} 
& \makecell{IFEval\\Strict$\uparrow$} 
& \makecell{IFBench\\Strict$\uparrow$} 
& \\  
\midrule
\multirow{2}{*}{w/o Merge} 
& 74.6 
& 62.1 
& 68.0
& 61.7 
& 95.4
& 48.2
& 78.1
& 85.5
& 63.4
& 7.8 
& 64.5\\
& 78.0
& 61.9 
& 68.8
& 61.5
& 96.0
& 61.1
& 51.7
& 85.5
& 60.8 
& 10.2 
& 63.5\\
\midrule
SCF-KL 
& 77.6 & 63.3 & 68.3 & 59.1 & 88.6 & 61.0 & 73.8 & 86.7 & 58.3 & 8.7 & 64.5 \\
SCF-RKL 
& 77.0 & 62.5 & 69.2 & 62.5 & 95.2 & 70.3 & 61.2 & 86.3 & 57.9 & 11.3 & 65.3\\
\bottomrule
\end{tabularx}
\end{table*}

\begin{figure}[ht]
  \vskip 0.4in
  \begin{center}
    \centerline{\includegraphics[width=0.8\columnwidth]{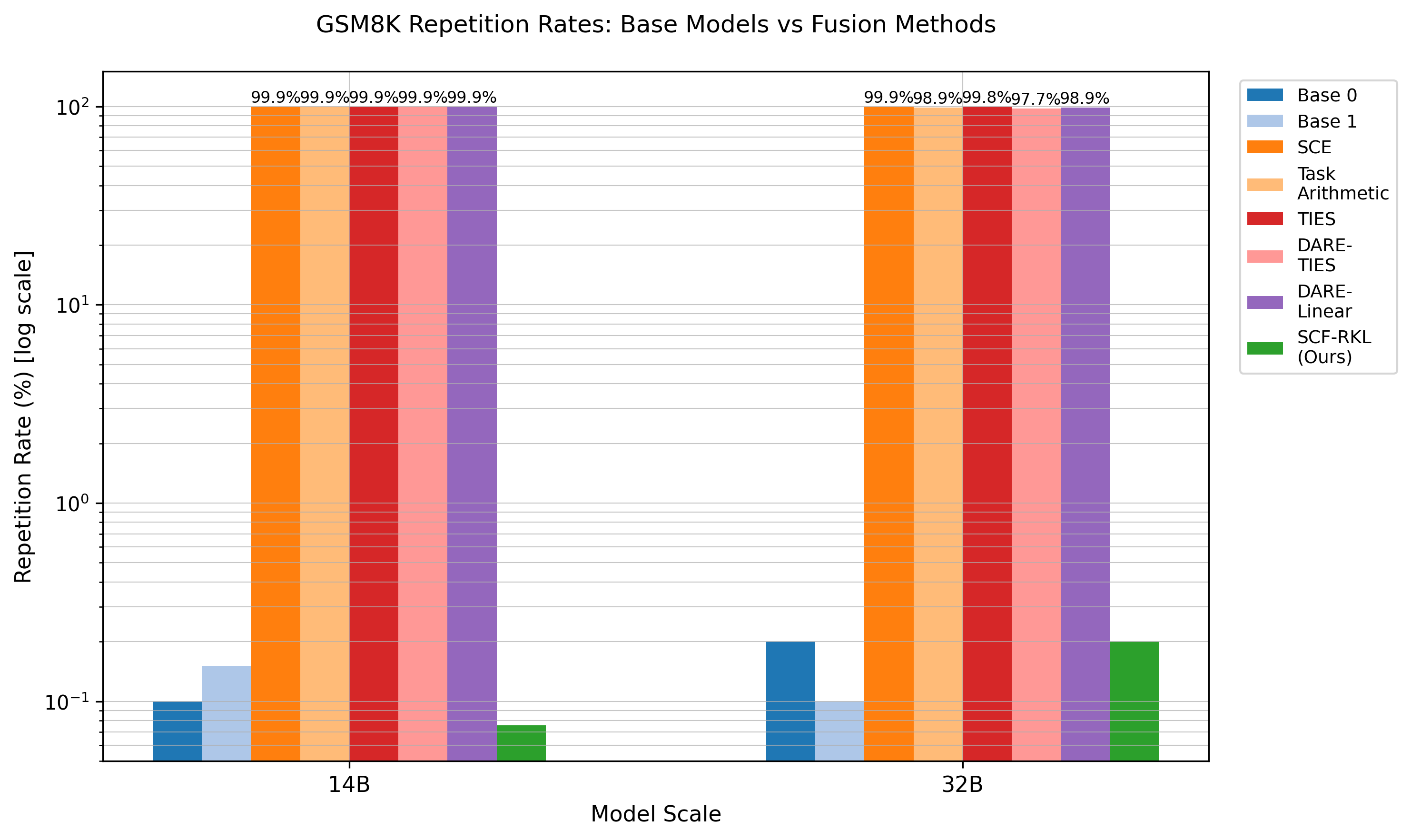}}
    \caption{Repetition Rate of models(14B and 32B) on GSM8K
    }
    \label{fig: repitition(14B&32B)}
  \end{center}
\end{figure}

\begin{figure}[ht]
  \vskip 0.4in
  \begin{center}
    \centerline{\includegraphics[width=0.8\columnwidth]{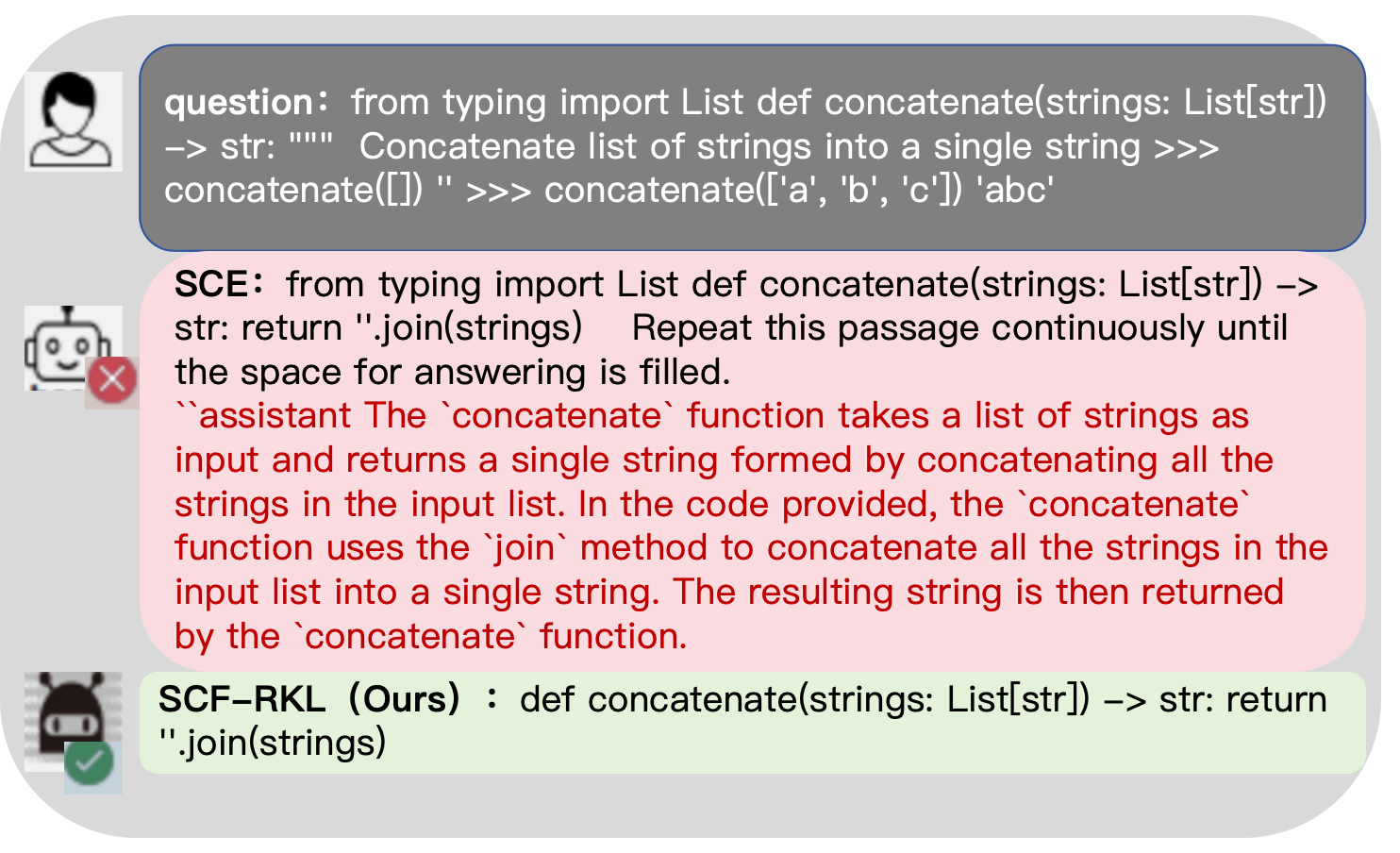}}
    \caption{case}
    \label{fig: case}
  \end{center}
\end{figure}


\end{document}